\documentclass{article}

\usepackage[preprint]{corl_2026}
\usepackage{silence}
\usepackage{booktabs}
\usepackage{tabularx}
\usepackage[most]{tcolorbox}

\tcbset{
    rewardbox/.style={
        colback=gray!3,
        colframe=gray!35,
        boxrule=0.4pt,
        arc=2mm,
        left=1.5mm,
        right=1.5mm,
        top=1mm,
        bottom=1mm,
        before skip=0.8em,
        after skip=0.8em
    }
}
\usepackage[utf8]{inputenc} %
\usepackage[T1]{fontenc}    %

\usepackage{url}            %
\usepackage{booktabs}       %
\usepackage{amsfonts}       %
\usepackage{nicefrac}       %
\usepackage{microtype}      %
\usepackage{xcolor}         %
\usepackage{graphicx}
\usepackage{adjustbox}
\usepackage{multirow}
\usepackage{rotating}

\usepackage{microtype}
\usepackage{graphicx}
\usepackage{subfigure}
\usepackage{booktabs} %

\usepackage{amsmath}
\usepackage{amssymb}
\usepackage{mathtools}
\usepackage{amsthm}

\newtheorem{proposition}{Proposition}[section]

\newtheorem{definition}{Definition}[section]
\newtheorem{assumption}{Assumption}[section]
\theoremstyle{remark}

\usepackage{url}
\usepackage{enumitem}
\usepackage{graphicx}
\usepackage{xcolor}
\usepackage{soul}  %
\usepackage{color, soul}
\usepackage{colortbl}
\usepackage{wrapfig}
\usepackage{algorithm}
\usepackage{multirow}
\usepackage{multicol}

\usepackage{lipsum} %

\usepackage{longtable}
\usepackage{listings}

\usepackage[toc,page,header]{appendix}
\usepackage{minitoc}

\usepackage[utf8]{inputenc} %
\usepackage[T1]{fontenc}    %
\usepackage[table]{xcolor}  %
\usepackage{hyperref}       %
\usepackage{url}            %
\usepackage{booktabs}       %
\usepackage{amsfonts}       %
\usepackage{amsmath}
\usepackage{nicefrac}       %
\usepackage{microtype}      %
\usepackage{multirow}
\usepackage{miscs/customs}
\usepackage{graphicx}
\usepackage{algorithm}
\usepackage{algpseudocode}
\usepackage{enumitem}
\usepackage{wrapfig}
\usepackage{caption}
\usepackage{needspace}
\usepackage{wrapfig}
\usepackage{makecell}
\usepackage[table]{xcolor}
\usepackage{makecell}
\usepackage{array}
\usepackage{placeins}
\definecolor{codegreen}{RGB}{226,242,226}

\usepackage[toc,page,header]{appendix}
\usepackage{minitoc}

\title{Support-Constrained RL Enables Real-World Policy Improvement without Real-World Experience}

\doparttoc

\author{%
  Raymond Yu\thanks{Equal contribution.} \\
  University of Washington
  \And
  William Huey\footnotemark[1] \\
  University of Washington
  \And
  Mustafa Mukadam \\
  University of Washington
  \And
  Anusha Nagabandi \\
  Amazon FAR
  \And
  Abhishek Gupta \\
  University of Washington
}

\hypersetup{
  pdftitle={Support-Constrained RL Enables Real-World Policy Improvement without Real-World Experience},
  pdfauthor={Raymond Yu, Will Huey, Mustafa Mukadam, Anusha Nagabandi, Abhishek Gupta},
  pdfsubject={},
}

\begin{document}
\begingroup
  \makeatletter
  \long\def\contentsline#1#2#3#4{}
  \def\tableofcontents{\@starttoc{toc}}
  \makeatother
  \tableofcontents
\endgroup

\maketitle

\begin{abstract}
Robots trained on real world data tend to be imprecise, slow, and brittle to perturbations. Improving these policies with reinforcement learning (RL) is an appealing alternative, but this process often requires expensive training in the real world. Performing policy improvement in simulation instead provides a far cheaper alternative, but unconstrained RL in simulation can exploit contact and dynamics mismatches, resulting in unsafe behaviors that do not transfer to hardware. Common forms of regularization can furthermore limit improvement by overconstraining to an imperfect behavior prior.  In this work, we propose Support-Constrained Off-Domain REinforcement (\texttt{SCORE}), a real-to-sim-to-real framework that constrains RL in simulation to the \textit{support} of a generative policy pretrained on real data. We instantiate this constraint through \textit{flow steering}, restricting \texttt{SCORE} to actions the base policy can already produce, which ensures transferable behaviors while maximizing policy improvement. Improving a policy with \texttt{SCORE} requires minimal effort: it learns from sparse rewards, avoids distillation, and leaves the base policy untouched. Across eight real-world dexterous multi-fingered robotic manipulation tasks, \texttt{SCORE} improves average success rate from $37.8\%$ to $89.9\%$, compared to $59.5\%$ for the best baseline, and reaches success in $36.8\%$ fewer steps than the base policy. Ultimately, through extensive experiments and ablations, we show that simulation can substantially improve real-world manipulation policies when policy optimization is appropriately constrained, introducing a new paradigm for real-to-sim-to-real policy improvement. Videos and code are available at \url{https://weirdlabuw.github.io/score/}.

\end{abstract}
\keywords{Reinforcement Learning, Sim-to-Real, Dexterous Manipulation}

\section{Introduction}
Real-world robot policies learned via imitation learning are rarely optimal zero-shot. Even when they can complete a task, they are often slow, imprecise, or brittle to perturbations. A natural next step is to improve these policies through interaction. In principle, a robot could repeatedly attempt a task, observe failures, and refine its behavior using reinforcement learning (RL). In practice, doing this directly on hardware is expensive, unsafe, and difficult to scale.

The difficulty of real-world improvement suggests the need for a different source of interaction. An appealing alternative to real-world RL is to improve the policy using simulation, which provides safe, massively parallel interaction, privileged state, and many opportunities for failure and recovery that are difficult to obtain on hardware \cite{yin2026emergentdexteritydiverseresets}. However, policy improvement without constraints in simulation can exploit discrepancies in contacts, object motion, or low-level control, achieving high simulated reward but failing on hardware deployment \cite{aljalbout2025realitygaproboticschallenges}. 

\begin{figure*}
    \centering

    \makebox[\textwidth][c]{%
        \includegraphics[
            width=1.0\textwidth
        ]{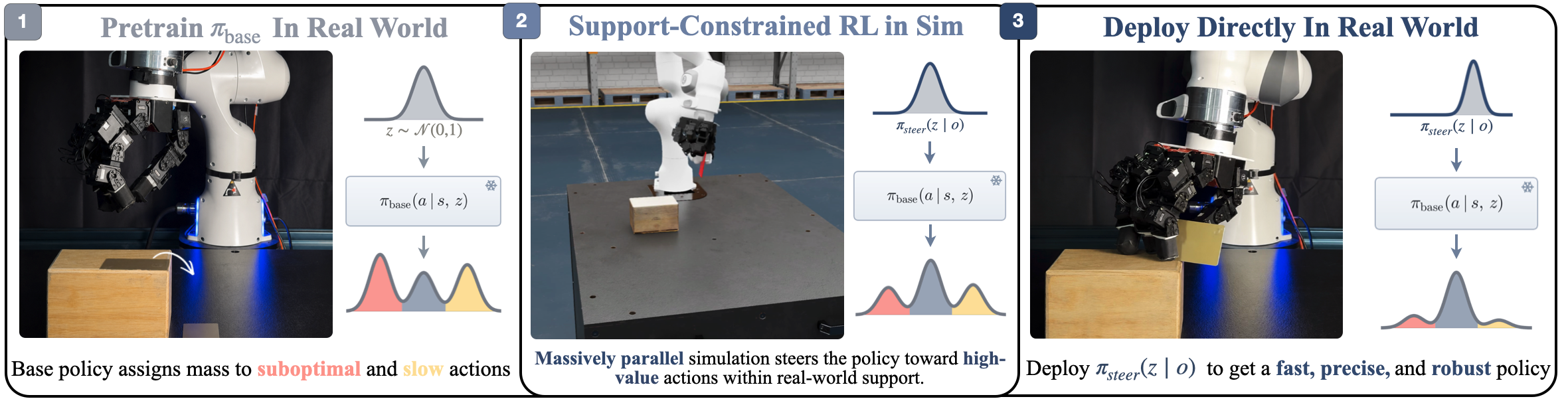}
    }
    \caption{\footnotesize{\textbf{\ours{} framework.} \ours{} starts from any real-world flow matching policy, which may have been trained on successes, play data, failures, and retry behaviors. The flow policy is brought into simulation, where \ours{} learns to improve the policy using flow steering, a support-constrained RL algorithm. Finally, our training framework enables direct deployment of the steering policy in the real world, preserving the base policy and avoiding complicated distillation pipelines.}}
    \label{fig:method}
    \vspace{-1.6em}
\end{figure*}

To ensure transferability of simulation policies to the real world, a common strategy constrains policy improvement with a distributional distance. For example, a behavior cloning (BC) loss can be added during policy gradient updates to penalize the KL divergence from the base policy~\cite{torne2024reconcilingrealitysimulationrealtosimtoreal}. While such constraints can improve transfer, they induce a challenging tradeoff
between policy improvement, which requires a relaxed constraint, and transferability, which requires a tight constraint. To avoid this tradeoff, rather than constraining the improved policy to match the \textit{distribution} of the base policy, we constrain
it to act within the base policy's \emph{support}: the set of actions that the pretrained base policy can generate for a given observation. Under this constraint, policy improvement can amplify faster, more precise, or more robust behaviors supported by the base policy while suppressing suboptimal actions. This forces simulation RL to perform actions already represented in the real-world data, preventing it from introducing arbitrary actions that may not transfer to the real world. 

\textbf{Our key insight is that policy improvement in simulation should be constrained to the \emph{support} of the real-world base policy.} We instantiate this constraint with \ours{}: Support-Constrained Off-Domain REinforcement, a real-to-sim-to-real framework for policy improvement. \ours{} uses flow steering in simulation to achieve support-constrained policy improvement. In the real-to-sim stage, we train a base policy from real world demonstrations via generative imitation learning and construct simulated task environments from real-world scans. In simulation, we freeze the base policy and learn a \emph{steering policy} over its latent inputs using sparse rewards. In the final sim-to-real stage, we combine our steering policy with the real-world base policy, and deploy it in the real world, without distillation or finetuning of the base policy. We show that this yields a simple and scalable paradigm for improving policies in complex contact-rich domains, such as dexterous multi-fingered manipulation. 

Concretely, our contributions are:
\vspace{-6pt}
\begin{enumerate}[itemsep=1pt]
    \item \textbf{A formalization of off-domain policy improvement}, where the real-world performance of a pretrained policy must be improved using interaction only in a related source domain. We show why distributional constraints end up preserving suboptimal modes of the
    base policy, while support constraints allow for flexible improvement within the real-world prior.
    \item \textbf{Support-Constrained Off-Domain REinforcement (\ours)}, a practical
    support-constrained framework for improving real-world manipulation policies in simulation that preserves the base policy, avoids distillation, and learns from sparse rewards.
    \item \textbf{Real-world experiments across eight
    multi-fingered manipulation tasks}, where \ours{} improves average success rate
    from $37.8\%$ to $89.9\%$ and decreases average execution time per task by $36.8\%$. Furthermore, our ablations analyze the role of pretraining data diversity, simulation design choices, and the limitations of support constraints.
\end{enumerate}

\section{Related Work}
\textbf{Policy Learning for Dexterous Manipulation:}
Dexterous manipulation with multi-fingered hands remains challenging due to high-dimensional control, partial observability, and contact-rich dynamics. One line of work learns dexterous manipulation policies directly from demonstrations, often using expressive generative policy classes such as diffusion models and flow matching to represent multimodal action distributions~\cite{ho2020denoisingdiffusionprobabilisticmodels, chi2024diffusionpolicyvisuomotorpolicy, yan2025maniflowgeneralrobotmanipulation, intelligence2025pi05visionlanguageactionmodelopenworld}. While these policies provide strong behavioral priors, they are trained to reproduce the data distribution and can inherit slow, imprecise, or brittle behaviors from the demonstrations. Recent algorithmic work has therefore studied how to improve diffusion and flow-based policies with reinforcement learning~\cite{ren2024diffusionpolicypolicyoptimization, mcallister2025flowmatchingpolicygradients, park2025flowqlearning, wagenmaker2025steeringdiffusionpolicylatent, su2026rfsreinforcementlearningresidual, hong2026tmrldiffusiontimestepmodulatedpretraining, yi2026flowpolicygradientsrobot}.

\textbf{Simulation-Based Dexterous Policy Learning:}
A complementary line of work uses simulation to learn dexterous skills with RL, often relying on privileged state, domain randomization, reset distributions, system identification, dynamics adaptation, policy distillation, or real-world fine-tuning to transfer policies to hardware~\cite{yin2026emergentdexteritydiverseresets, torne2024reconcilingrealitysimulationrealtosimtoreal, yin2025dexteritygenfoundationcontrollerunprecedented, memmel2024asidactiveexplorationidentification, kumar2021rmarapidmotoradaptation, liu2025dexndmclosingrealitygap}. Other approaches reduce the sim-to-real and exploration burden through structured action spaces, state-estimation modules, retargeting, or object-centric representations~\cite{mandi2025dexmachinafunctionalretargetingbimanual, chen2025vividexlearningvisionbaseddexterous, qin2022dexmvimitationlearningdexterous, qin2022dexpointgeneralizablepointcloud, kedia2026simtoolrealobjectcentricpolicyzeroshot}. More directly related to our setting are methods that use imperfect behavior priors as the starting point for simulation optimization. ExpertGen trains a diffusion policy from behavior priors such as human- or LLM-generated scripted trajectories collected in \textit{simulation}, and steers the diffusion noise in simulation to obtain high-success state-based expert policies~\cite{xu2026expertgenscalablesimtorealexpert}. Although ExpertGen shares conceptual similarities with \ours{}, it starts from behavior priors generated in simulation and trains state-based policies that are distilled for real-world deployment, similar to \citet{su2026rfsreinforcementlearningresidual}. We instead study off-domain policy improvement, where a prior learned from real-world data is improved through simulation and reused directly for deployment. This shifts the role of simulation from generating a new expert policy to improving an existing real-world behavior prior, avoiding distillation, real-world fine-tuning, or modification of the base policy.

\textbf{Constrained Policy Improvement:}
When policy improvement relies on offline data or simulation rather than real-world interaction, unconstrained optimization can exploit poorly covered actions or simulator-specific artifacts, failing to transfer to the real world. This has been addressed in simulation-based policy learning using domain classifiers~\cite{eysenbach2021offdynamicsreinforcementlearningtraining}, motion priors~\cite{Peng_2021, dan2025xsimcrossembodimentlearningrealtosimtoreal}, behavior-cloning regularization~\cite{torne2024reconcilingrealitysimulationrealtosimtoreal}, or dynamics-aware penalties~\cite{niu2022when}. Parallel work in offline RL penalizes high values for actions outside the offline data distribution by learning conservative value estimates~\cite{wu2019behaviorregularizedofflinereinforcement, kumar2020conservativeqlearningofflinereinforcement}. Rather than only regularizing toward the behavior distribution or pessimistically modifying values, recent work studies \emph{support-constrained} improvement, where policy optimization is restricted to actions represented in the dataset~\cite{wagenmaker2025steeringdiffusionpolicylatent, singh2022offlinerlrealisticdatasets, mao2023supportedtrustregionoptimization, zhang2026reformreflectedflowsonsupport}. We build on this support-constrained perspective in the off-domain setting, where online interaction is available, but only in simulation.

\section{Preliminaries}

\subsection{Problem Setup}

We begin by defining the simulation and real-world domains as Markov Decision Processes (MDPs), denoted $\mathcal{M}_{\mathrm{sim}}$ and $\mathcal{M}_{\mathrm{real}}$, respectively. The two domains share a state space $\mathcal{S}$, action space $\mathcal{A}$, reward function $r:\mathcal{S}\times\mathcal{A}\rightarrow\mathbb{R}$, discount factor $\gamma \in (0,1)$, and initial state distribution $p_1(s_1)$. They differ only in their transition dynamics, $p_{\mathrm{sim}}(s_{t+1}\mid s_t,a_t)$ and $p_{\mathrm{real}}(s_{t+1}\mid s_t,a_t)$. We assume access to an initial policy $\pi_{\mathrm{base}}$, trained via interaction in the real-world domain. We assume minimal perceptual shift between the two domains under sufficient visual randomization.

For any policy $\pi$, we define its real-world performance using the expected discounted return:
\begin{equation}
J_{\mathrm{real}}(\pi)
=
\mathbb{E}_{\pi, p_{\mathrm{real}}}
\left[
\sum_{t=1}^{T} \gamma^{t-1} r(s_t,a_t)
\right],
\end{equation}
where trajectories are generated under the real-world transition dynamics
$p_{\mathrm{real}}$.

\begin{definition}[Off-Domain Policy Improvement]\label{def:odpi}
Given a policy $\pi_{\mathrm{base}}$ initially learned from interaction in $\mathcal{M}_{\mathrm{real}}$,  off-domain policy improvement is the problem of finding a policy $\hat{\pi}$ such that $J_{\mathrm{real}}(\hat{\pi}) > J_{\mathrm{real}}(\pi_{\mathrm{base}})$, using additional interaction only in $\mathcal{M}_{\mathrm{sim}}$.
\end{definition}

In dexterous manipulation, off-domain policy improvement is especially challenging because small discrepancies in contact dynamics, physical parameters, or low-level control can determine whether a behavior succeeds after deployment. As a result, there are many policies that achieve high performance in simulation, but completely fail when deployed in the real world.

\subsection{Illustrative Example of Off-Domain Policy Improvement}\label{prelim:toy_example}

To illustrate the problem, Fig.~\ref{fig:toy_example} instantiates a simple 2D example in which a robot arm is tasked with reaching a lightbulb. In the real world, a barrier prevents the arm from moving straight down to reach the lightbulb, but this barrier is unmodeled in the simulator, as indicated by its translucent shading. The barrier is representative of realistic failure cases in manipulation settings, such as small artifacts in the manipulator geometry that cause collisions in the real world but are unmodeled in the simulator. We assume access to a base policy that properly avoids the barriers and achieves non-zero success at a task, but is slow and imprecise, as indicated by the warning sign (Fig.~\ref{fig:toy_example}, first panel). Our goal is to improve this policy with simulated interaction.

\begin{figure*}[t]
    \centering

    \makebox[\textwidth][c]{%
        \includegraphics[
            width=\textwidth
        ]{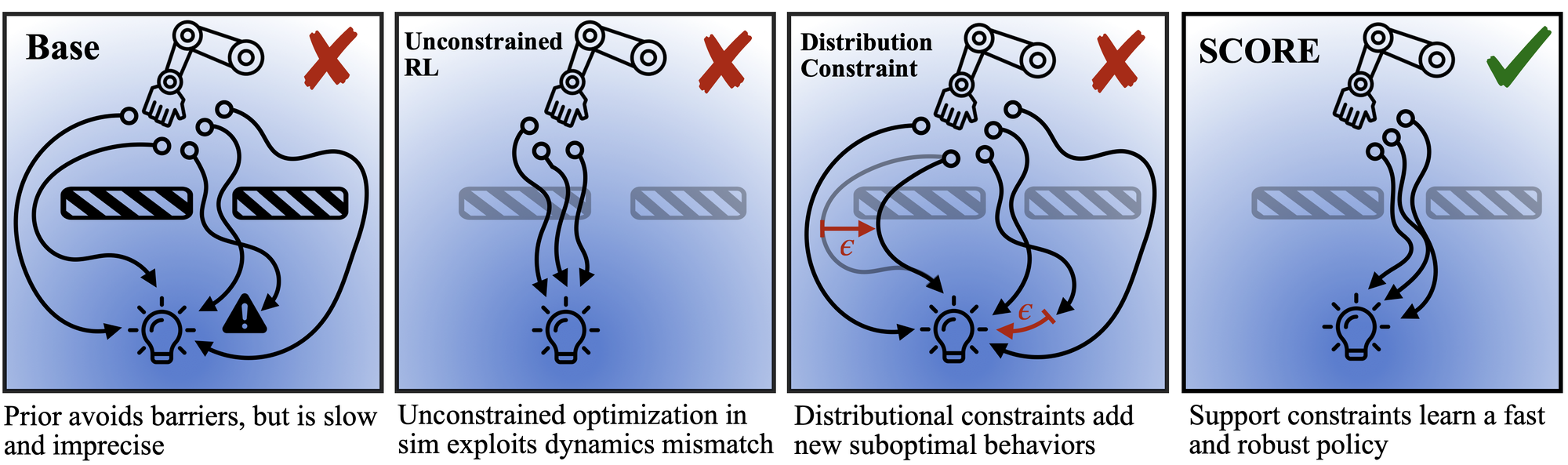}
    }
    \vspace{-1.1em}
    \caption{\footnotesize{\textbf{Toy Example}. The real-world base policy avoids barriers, but performs roundabout trajectories that sometimes miss the goal. In simulation, unconstrained RL exploits dynamics mismatch to move directly towards the goal, but this fails in the real world. As shown by the red arrows, distributional regularization allows for small deviations from the base policy, refining imprecisions but preserving slow motion and occasionally causing new failure modes. Meanwhile, \ours{} learns a fast, robust, and transferable policy. }}
    \label{fig:toy_example}
    \vspace{-1.4em}
\end{figure*}

\textbf{Unconstrained RL:} A first attempt at solving this problem might ignore the base policy altogether, using RL to optimize the task level reward in simulation without any constraints (Fig.~\ref{fig:toy_example}, second panel). This can result in \textit{reward hacking}, where the policy achieves high reward in simulation but performs undesirable or dangerous behaviors \citep{amodei2016concrete,skalse2022defining}. In practical sim-to-real settings, reward hacking is often mitigated through hand-designing the reward or the environment to avoid these behaviors, which requires tedious cycles of policy training, deployment, and reward tuning \citep{tan2018sim,wu2021learning}.

\textbf{Distributionally constrained RL:} Instead of hand designing reward functions to solve the task, previous work aims to take advantage of the real-world prior \cite{torne2024reconcilingrealitysimulationrealtosimtoreal, eysenbach2021offdynamicsreinforcementlearningtraining, Peng_2021,  huey2025imitationlearningsingletemporally}. This typically involves regularizing the learned policy under a \textit{distributional distance} from the base policy (Fig.~\ref{fig:toy_example}, third panel). Formally, for a distributional distance $D$ and a constant hyperparameter $\epsilon$, the distributionally-constrained RL objective is:
\begin{equation}
\hat{\pi} = \arg\max_{\pi} J_{\mathrm{sim}}(\pi) \quad \text{s.t.} \quad D(\pi,  \pi_{\mathrm{base}}) < \epsilon
\end{equation}

Following the view of \cite{ke2020imitationlearningfdivergenceminimization}, many regularizers can be interpreted as imposing distributional closeness to a behavior policy: behavior-cloning regularization corresponds to the forward-KL divergence, adversarial imitation methods such as GAIL correspond to Jensen-Shannon regularization~\cite{ho2016generativeadversarialimitationlearning, Peng_2021}, and bounded residual policies correspond to optimization under a Wasserstein-1 constraint around the base policy. However, there is a fundamental problem with distributional regularization: it introduces a challenging tradeoff between policy improvement, which may require ignoring suboptimal modes of the base policy, and satisfying the constraint by staying near the base policy.

\textbf{Support constrained RL:}
Unconstrained optimization learns actions that fail on deployment, and distributional constraints overconstrain the policy to slow and imprecise behaviors. An ideal approach would find the optimal modes of the base policy, while ignoring slow and imprecise actions (Fig.~\ref{fig:toy_example}, fourth panel). This idea is formalized via a \textit{support constraint}, which restricts policy learning to the set of actions with nonzero likelihood under the base policy. Support-constrained policies can choose the fastest mode of the base policy, refine the imprecise failure by borrowing actions from another trajectory, and avoid untransferable behavior. Formally, the support-constrained RL objective is:
\begin{equation}\label{eq:supp_constrained_objective}
\hat{\pi} = \arg\max_{\pi} J_{\mathrm{sim}}(\pi) \quad \text{s.t.} \quad \mathrm{supp}(\pi) \subseteq  \mathrm{supp}(\pi_{\mathrm{base}})
\end{equation}

A similar objective has been proposed in offline RL, where the learned policy can exploit out-of-distribution errors in the value function if it deviates too far from the support of the offline dataset \cite{singh2022offlinerlrealisticdatasets}. Such approaches typically require complicated strategies involving discriminators, modified Q-values, and iterations of simulation and real world experience, introducing additional hyperparameters and sources of instability \cite{eysenbach2021offdynamicsreinforcementlearningtraining, niu2022when}. Instead, we satisfy support constraints using the simple framework of \textit{flow steering}, where a policy is learned over the noise space of a pretrained flow policy \cite{wagenmaker2025steeringdiffusionpolicylatent, zhang2026reformreflectedflowsonsupport}. We elaborate on flow steering in Section~\ref{sec:score}, providing a practical algorithm for online support-constrained RL in massively parallel simulation that enables direct improvement of real-world manipulation policies without requiring any further real-world interaction.

In Appendix~\ref{app:proofs}, we provide a principled analysis of the limitations of unconstrained RL and distributional constraints for off-domain policy improvement, formalizing the assumptions necessary to ensure that support constraints address these limitations.

\section{Support Constraints Enable Transferable Policy Improvement}
\label{sec:score}
\label{sec:method}

We now describe \ours{}, a real-to-sim-to-real framework that improves a
real-world policy entirely in simulation, without new real-world data, reward
engineering, or finetuning the original policy. In doing so, we show how simple and effective massively parallel simulation can be for improving real-world policies. The central design choice in \ours{} is to optimize task
performance in simulation while constraining the learned policy to the
\textit{support} of a generative policy trained on real-world data. Rather than
updating this policy directly, \ours{} freezes the real-world base policy
$\pi_{\mathrm{base}}$ and learns a lightweight steering policy
$\pi_{\mathrm{steer}}$ over its latent inputs (Figure~\ref{fig:method}). RL can
then amplify the high-reward behaviors already within the support of the
real-world prior, producing improvements that can transfer back to hardware with no further distillation steps. 

\subsection{Learning and Steering a Real-World Flow Policy}
\label{sec:flow_steering}

For each task, we train a conditional flow policy $\pi_{\mathrm{base}}(a \mid o)$ on
real-world robot data. Let $o$ denote the deployment observation (point cloud, proprioception, history) and $a$ the robot action. Flow matching~\cite{lipman2023flowmatchinggenerativemodeling} learns a velocity field
$v_\theta(x_\tau,\tau,o)$ transporting samples from a prior $p_0$ (e.g.\
Gaussian) to the behavior distribution. Given
$(o,a)\sim\mathcal{D}_{\mathrm{real}}$, $x_0\sim p_0$,
$\tau\sim\mathcal{U}[0,1]$, and interpolant
$x_\tau=(1-\tau)x_0+\tau a$, the training objective is:
\begin{equation}
    \mathcal{L}_{\mathrm{FM}}(\theta)
    =\mathbb{E}_{\tau,(o,a),x_0}\!\left[
    \left\|v_\theta(x_\tau,\tau,o)-(a-x_0)\right\|^2\right]
\end{equation}
At inference, we draw $z\sim p_0$ and integrate
$dx_\tau/d\tau=v_\theta(x_\tau,\tau,o)$ from $x_0=z$ to obtain
$a=x_1$, written $a=\pi_{\mathrm{base}}(o,z)$. After training, \ours{} freezes
$\pi_{\mathrm{base}}$ and improves the policy only by learning a steering policy
$\pi_{\mathrm{steer}}(z\mid o)$ over the latent input. At policy timestep $k$,
the resulting policy samples:
\[
    z_k \sim \pi_{\mathrm{steer}}(\cdot \mid o_k), \qquad
    a_k = \pi_{\mathrm{base}}(o_k,z_k)
\]
For pure latent steering, this constrains the policy to act through the frozen real-world prior: for a fixed
observation $o$, it can only produce actions in the \textit{model-induced} set
$\mathcal{A}_{\mathrm{base}}(o)
= \{\pi_{\mathrm{base}}(o,z)\mid z\in\mathcal{Z}\}$.
We use $\mathcal{A}_{\mathrm{base}}(o)$ as the support for steering. Latent steering performs support-constrained RL by changing which actions within that support are more likely. As a result, simulation can favor faster, more precise, or more robust behaviors already captured by the base policy, while limiting the ability of RL to introduce arbitrary simulator-specific actions. We instantiate this steering approach using two flow steering methods, DSRL~\cite{wagenmaker2025steeringdiffusionpolicylatent} and RFS~\cite{su2026rfsreinforcementlearningresidual}, originally developed for real-world online RL, but instead we use them to steer entirely in simulation. We refer to their use in our method as \oursdsrl{} and \ours{}, respectively. DSRL performs \emph{pure latent steering}: it optimizes only the flow noise $z$, so every action lies within the model-induced set $\mathcal{A}_{\mathrm{base}}(o)$ above, imposing a hard model-induced support constraint. RFS additionally adds a small residual $a_r$ after sampling from the base flow, yielding $a_k=\pi_{\mathrm{base}}(o_k,z_k)+a_r$. This can be viewed as a soft support constraint that locally refines actions around the base policy support, giving the precision needed for dexterous manipulation. See Appendix~\ref{app:proofs} for a formal discussion. More generally, \ours{} is modular over the choice of support-constrained method, where DSRL and RFS are just two instances \cite{zhang2026reformreflectedflowsonsupport}.

\subsection{Simulation Optimization and Direct Deployment}
\label{sec:sim_optimization}
A key aim of \ours{} is to make simulation improvement simple and
scalable with \textbf{minimal reward engineering}: we train all our tasks with \textbf{sparse rewards} defined by the same success condition used for evaluation and \textbf{one set of
hyperparameters}. For each task, we construct a simulated environment from a scanned real-world
scene and randomize object initializations and physical parameters during
training. Within this environment, we optimize $\pi_{\mathrm{steer}}$ while
keeping the base policy $\pi_{\mathrm{base}}$ fixed:
\begin{equation}
    \underset{\pi_{\mathrm{steer}}}{\max}\;
    J_{\mathrm{sim}}(\pi_{\mathrm{base}} \circ \pi_{\mathrm{steer}})
    =
    \mathbb{E}_{z_k \sim \pi_{\mathrm{steer}}(\cdot \mid o_k),
    \; a_k = \pi_{\mathrm{base}}(o_k,z_k),
    \; p_{\mathrm{sim}}}
    \left[
        \sum_{k=1}^{T} \gamma^{k-1} r(s_k,a_k)
    \right],
\end{equation}
Naively optimizing from sparse rewards can be sample inefficient, especially in
dexterous manipulation where successful contact sequences are rare. To make
better use of parallel simulation, we optimize $\pi_{\mathrm{steer}}$ with PPO~\cite{schulman2017proximalpolicyoptimizationalgorithms}
using an asymmetric actor-critic architecture. The actor receives only deployment
observations and outputs the latent steering distribution, while the critic uses
privileged simulator state only during training. This improves value estimation, speeds up policy learning, and makes use of privileged information available in simulation without changing the deployed policy. We remove the need for policy distillation or fine-tuning in the real world by sampling $(z_k, a_r)\sim\pi_{\mathrm{steer}}(\cdot\mid o_k)$ and
executing $a_k=\pi_{\mathrm{base}}(o_k,z_k) + a_r$.

\section{Experiments}
\begin{figure*}[t]
    \centering

    \makebox[\textwidth][c]{%
        \includegraphics[
            width=\textwidth
        ]{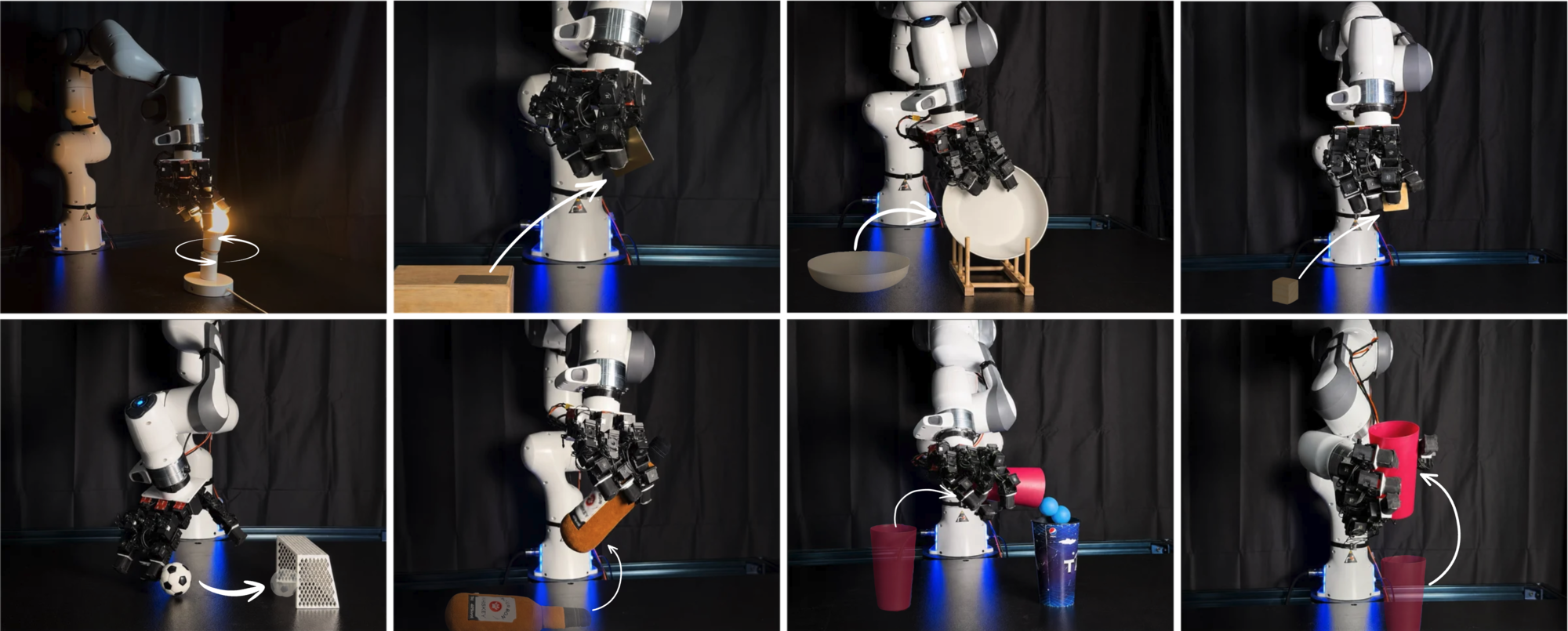}    }
    \vspace{-1.0em}
    \caption{\footnotesize \textbf{Real-world tasks.} We evaluate on eight contact-rich dexterous manipulation tasks spanning grasping, pouring, pushing, reorientation, and object placement.}
    \vspace{-1.4em}
    \label{fig:all_tasks}
\end{figure*}

We study whether support-constrained policy improvement in simulation can
improve real-world multi-fingered dexterous manipulation policies. This setting is especially challenging because the high-dimensional action space makes exploration difficult, and the complexity of the end-effector leads to unmodeled dynamics. Without constraints or extensive reward shaping, this often leads to learning unnatural and untransferable finger motions in simulation, even for simple grasping tasks. Our experiments are designed to answer three questions:
(1) what is the importance of support constraints in off-domain policy improvement; (2) how do the quality and coverage of the
real-world data affect the gains achievable by \ours{}; and (3) what system design
choices are necessary for a simple and efficient real-to-sim-to-real pipeline?  

Our evaluation consists of eight real-world dexterous manipulation tasks using a LEAP Hand \cite{shaw2023leaphandlowcostefficient} attached to a Franka FR3 arm, as shown in Fig.~\ref{fig:all_tasks}. In Appendix~\ref{app:addl_exps} and~\ref{app:qualitative} we report per-task evaluation metrics and qualitative examples respectively. Additional
details on data collection, rewards, evaluation protocol, and baseline implementations are provided in Appendix~\ref{app:exp_details}, and details on environment creation and domain randomization in Appendix~\ref{app:sim_details}.

\begin{wrapfigure}[14]{r}{0.5\textwidth}
    \centering
    \vspace{-1.5em}
    \includegraphics[width=\linewidth]{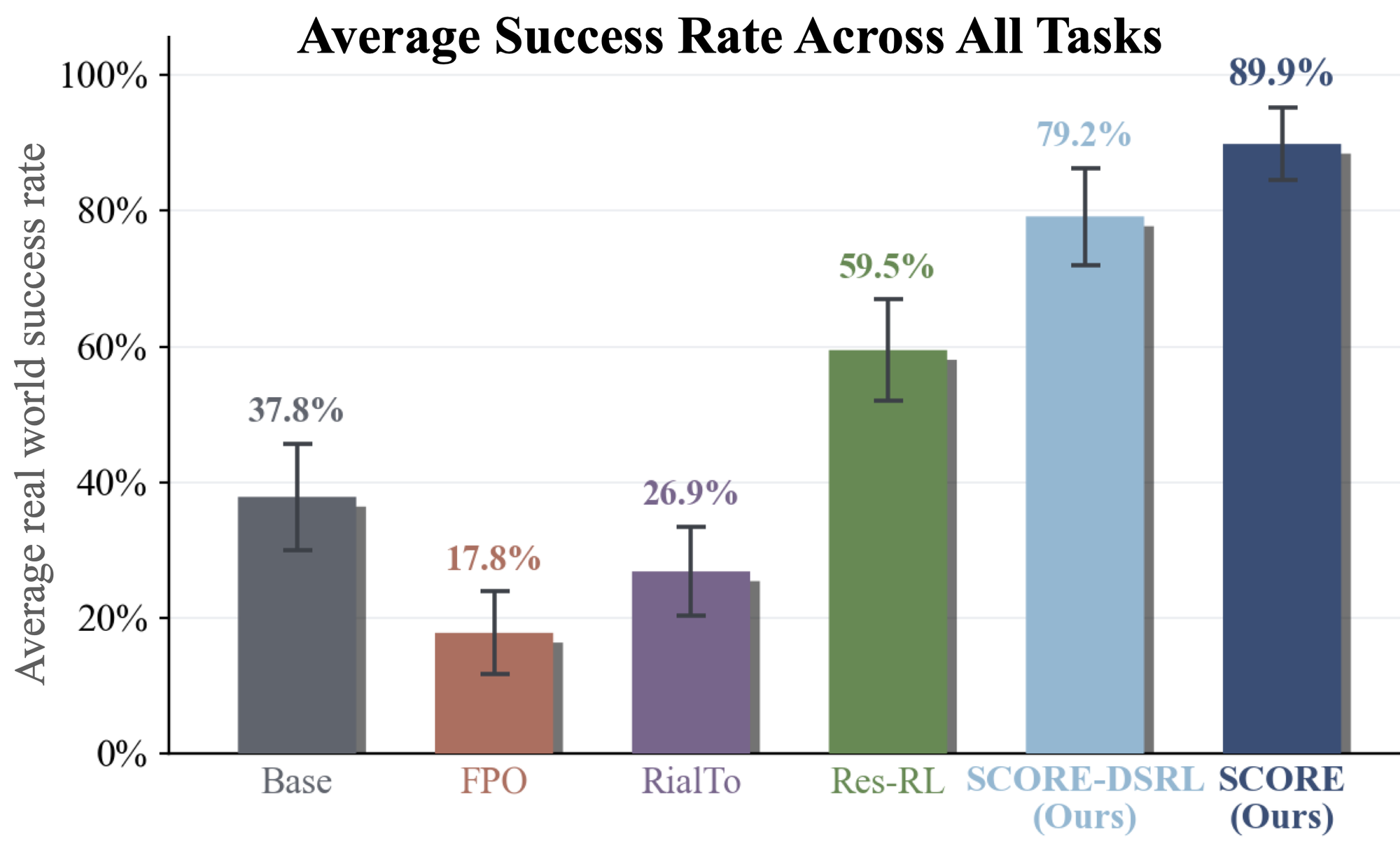}
    \vspace{-1.5em}
    \caption{\footnotesize \textbf{Average real-world success rate across all 8 tasks}. \ours{} and \oursdsrl{} outperform all baselines, while \fpo{} and \rialto{} learn dangerous actions, and \res{} is constrained to suboptimal behaviors.}
    \label{fig:real_success_bar}
    \vspace{-.5em}
\end{wrapfigure}

\subsection{Can \ours{} improve real-world policies using simulation?}

Figure~\ref{fig:real_success_bar} reports average real-world success rate across all eight tasks. \ours{} improves over the base policy on every task, increasing average success rate from \textbf{37.8\% to
89.9\%}. We report individual task success in Appendix~\ref{app:task_success_table}. \ours{} brings all easy tasks to 100\%, and yields its largest gains over the base policy on Ball Pour ($11\%\rightarrow89\%$), Credit Card Pick ($10\%\rightarrow80\%$), and Cube Pinch ($30\%\rightarrow100\%$). Ball Pour and Credit Card Pick require precise motions that the base policy cannot consistently perform, while \ours{} learns precision through massively parallel interaction. In Cube Pinch, \ours{} learns to consistently recover from failed grasps, enabling $100\%$ success rates. Consistent with \citet{su2026rfsreinforcementlearningresidual}, \ours{} outperforms \oursdsrl{}, though both exhibit similar behavior and perform better than all baselines. To qualitatively substantiate our claims, we provide visuals of \ours{} and all baselines in Appendix \ref{app:qualitative}.
\begin{wrapfigure}[13]{r}{0.32\textwidth}
    \centering
    \vspace{-1em}
    \includegraphics[width=\linewidth]{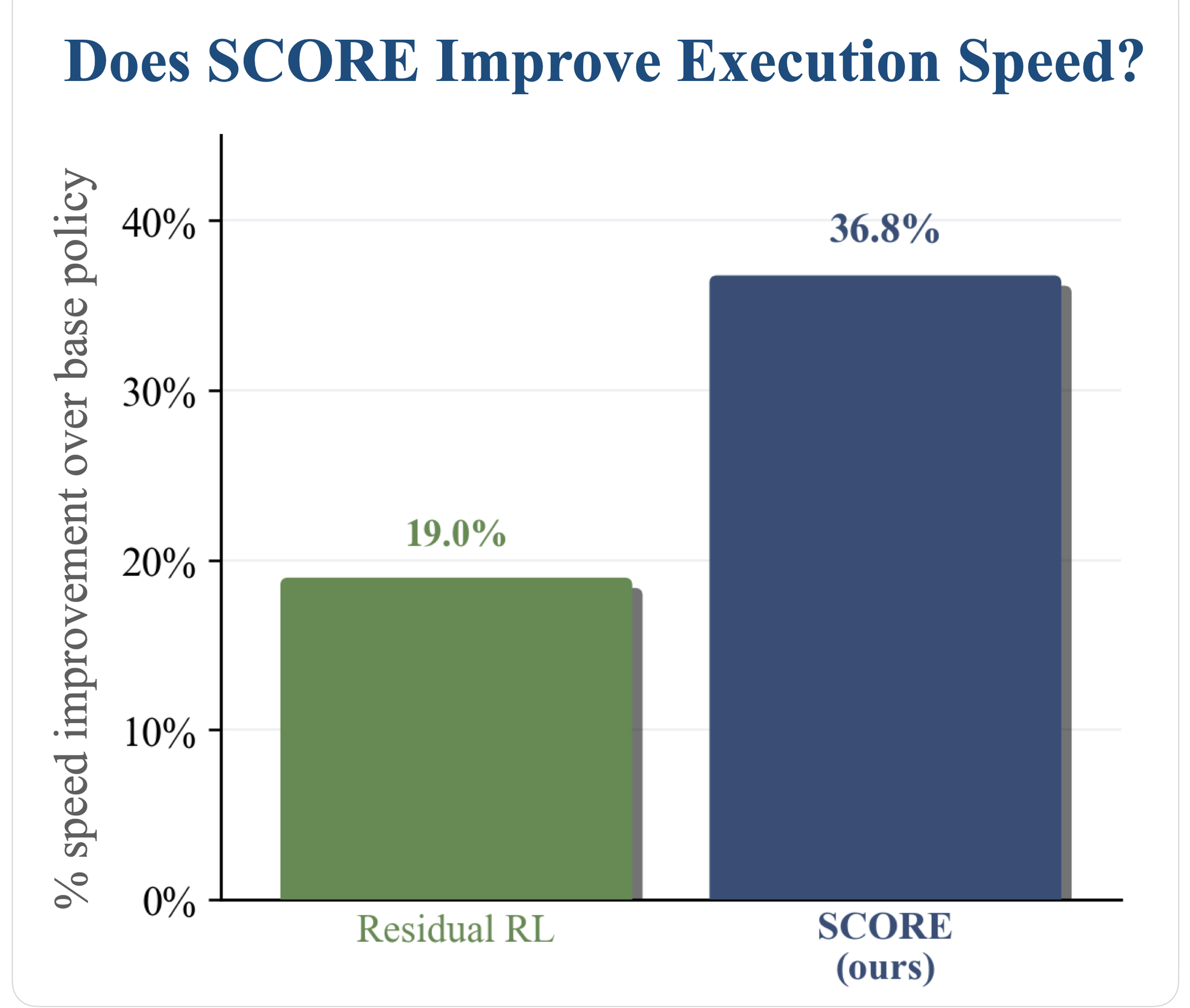}
    \caption{\footnotesize \textbf{Speed improvement} of \ours{} and \res{} over \base, averaged across 8 tasks.} 
    \vspace{-3em}
    \label{fig:real_speed_bar}
\end{wrapfigure}

Motivated by our discussion in Section~\ref{prelim:toy_example}, we now empirically investigate how support constraints overcome the limitations of unconstrained and distributionally constrained optimization. 

\textbf{Does unconstrained optimization in simulation result in dangerous behavior?} To test our hypothesis about unconstrained optimization, we optimize the base policy in simulation using \fpo~\cite{mcallister2025flowmatchingpolicygradients}, a state-of-the-art method for reinforcement learning with flow policies. As shown in Appendix~\ref{app:task_success_table}, FPO policies achieve high performance in simulation, but their performance drastically degrades when deployed in the real world due to dangerous reward hacking behavior. Appendix~\ref{app:qualitative} illustrates this behavior most prominently in the Soccer Push task, where \fpo{} learns to strike the table hard with the hand and trap the ball using simulator-specific contact dynamics. While this behavior succeeds in simulation, it is unsafe to deploy on real hardware.

\textbf{Do distributional constraints limit optimal improvement?}  To test whether distributional constraints can indeed preclude optimal improvement, we compare \ours{} against two distributionally-regularized baselines, \rialto{} and \res{}. \rialto{} is a real-to-sim-to-real method that performs PPO with a BC regularization term, equivalent to forward-KL regularization. As shown in Fig.~\ref{fig:distributional_tradeoff_empirical}, we observe a clear tradeoff: weak regularization improves simulation performance but leads to poor real-world transfer, while strong regularization keeps the policy close to the data and limits improvement. Consequently, \rialto{} fails substantially relative to the base policy in the real world.

\res{} learns a bounded residual policy over the base policy, which can be interpreted as constraining policy improvement via the Wasserstein-1 distance. \res{} improves over the base policy on several tasks, indicating that a residual can be useful for local refinements. However, unlike \ours{}, \res{} can only locally perturb the action chosen by the base policy, limiting its improvement over a slow or suboptimal base policy. Fig.~\ref{fig:real_speed_bar} shows that, on average, \res{} completes tasks just 19.0\% faster than the base policy, whereas \ours{} completes tasks \textbf{36.8\% faster}. This indicates that \ours{} steers towards fast modes of the base policy, addressing a key conceptual limitation of distributional constraints discussed in Sec.~\ref{prelim:toy_example}. 

\begin{figure*}[t]
    \centering

    \makebox[\textwidth][c]{%
        \includegraphics[
            width=\textwidth
        ]{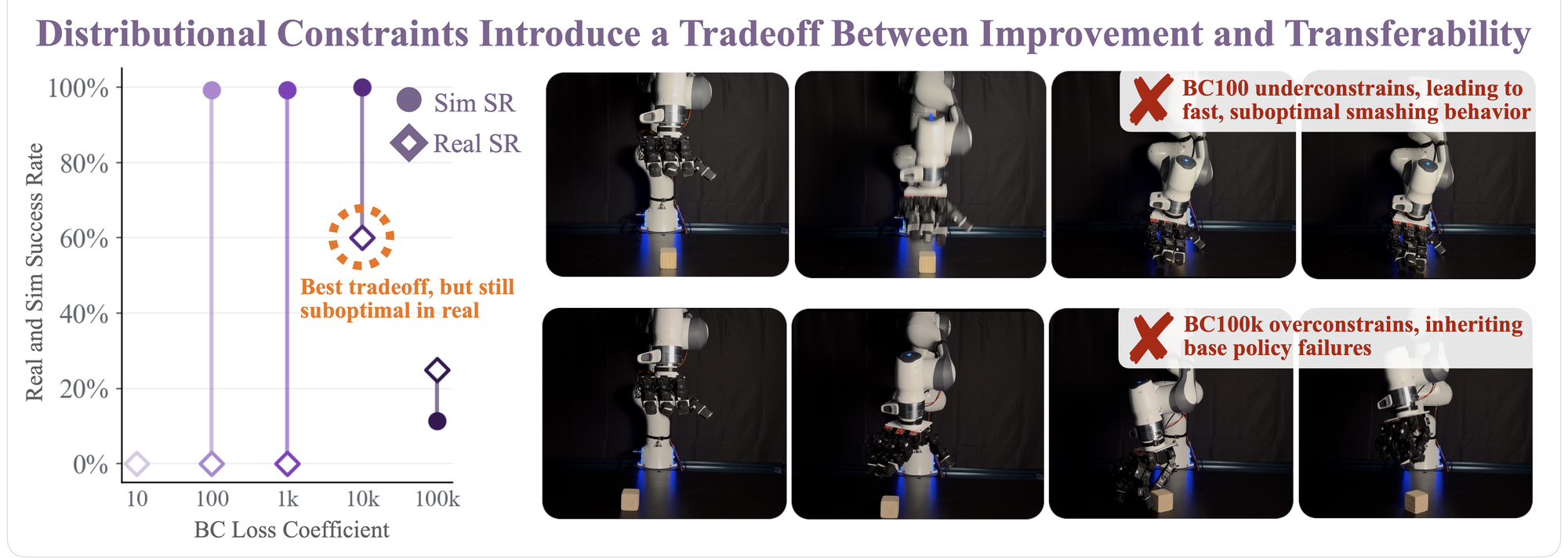}
    }
    \caption{\footnotesize{\textbf{Distributional Constraints Introduce a Tradeoff Between Improvement and Transferability}. The left plot shows the simulated performance (circles) and real world performance (diamonds) of RialTo policies trained with 5 different levels of BC regularization during BC-PPO. A value of 10 leads to collapse in simulation, while a larger value of 100 learns a dangerous strategy far from the base policy distribution, as shown in the top film strip. On the other end of the spectrum, a value of 100k constrains policy learning in simulation, retaining suboptimal failure modes of the base policy. The middle ground of 10k achieves $60\%$ success rate in the real world, improving over the base policy's 30\%, but remaining far below \ours{}, which achieves 100\%. }}
    \label{fig:distributional_tradeoff_empirical}
    \vspace{-1.5em}
\end{figure*}

\subsection{How does the coverage of the real-world base policy affect \ours{} improvement?}
\label{subsec:coverage_ablation}

\ours{} can only steer through the support of the base policy, so its improvement depends on the coverage of the real-world data. We study this along several axes, including the amount of data, imperfect retry and play data, adaptation to unseen objects and distractors, and a shared multi-task prior.

\begin{wrapfigure}[11]{r}{0.45\textwidth}
    \centering
    \vspace{-1.6em}
    \includegraphics[
        width=0.45\textwidth,
    ]{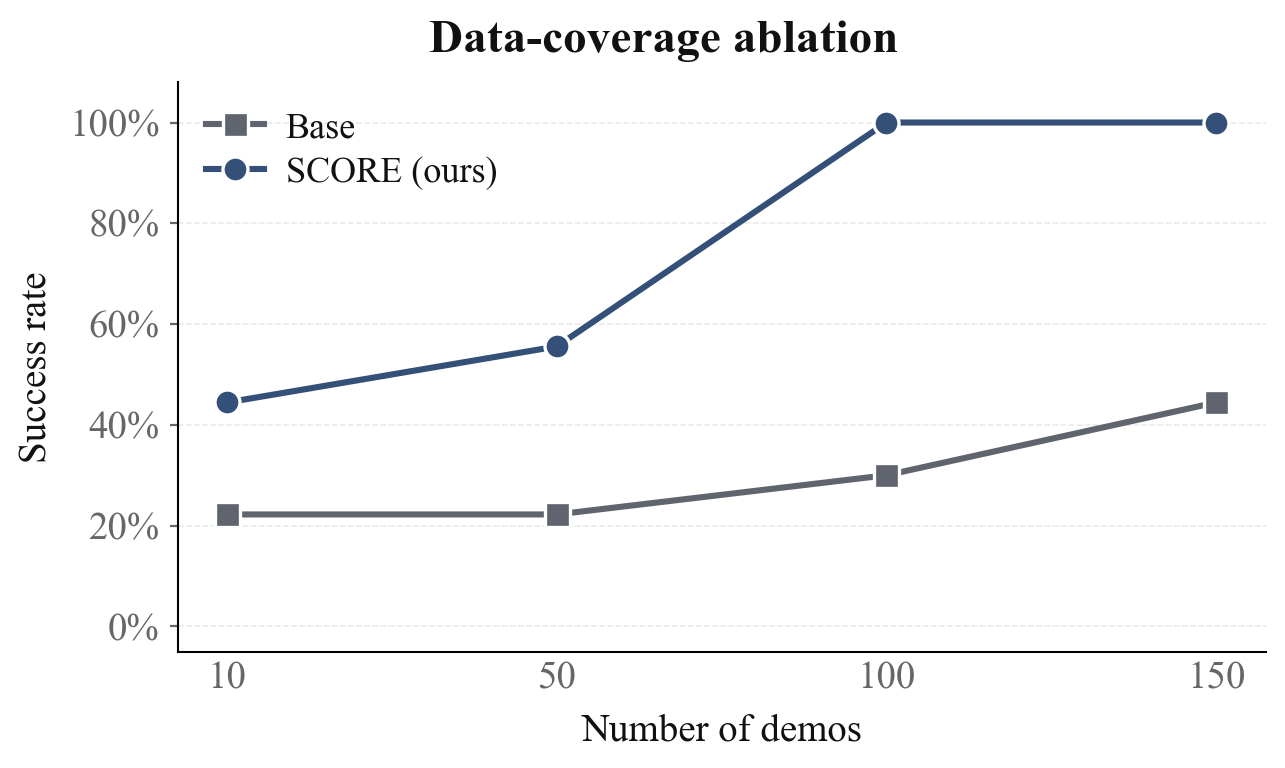}
    \caption{\footnotesize \textbf{Data-size ablation.}
    More data enables stronger support-constrained improvement.}
    \label{fig:data_ablation_curve}
    
\end{wrapfigure}

\textbf{Can more demonstrations improve steering?} We train base policies on the Cube Pinch task with varying numbers of demonstrations and apply \ours{} on each. Fig.~\ref{fig:data_ablation_curve} shows that additional demonstrations produce only modest improvements in the base policy. However, applying \ours{} on top of a base policy trained with more data improves sharply, reaching 100\% success with 100 demonstrations. Additional demonstrations may lead to stronger support, allowing \ours{} to amplify gains during post-training even when the pretrained policy improves marginally.

\textbf{Can imperfect data improve steering?}
\label{subsec:retry} To do so, we train two base policies on Cube Pinch datasets with equal size but different coverage: one contains 70 optimal demonstrations plus 40 additional optimal grasps, while the other contains the same 70 optimal demonstrations plus 40 retry demonstrations. Retry behaviors include intentional misses, drops, or perturbations of the cube before recovering. As shown in Fig.~\ref{fig:retry_steering}, retry data does not improve the base policy by itself. However, \ours{} applied to the base policy trained with retry data reaches $100\%$ success, up from $40\%$ when the base policy is trained on optimal demonstrations only. We observe the same principle with play data: we train one base policy with with 40 demonstrations on the right side of the workspace, and another with 10 additional suboptimal interactions on the left side of the workspace. These involve arm motions towards the object, but no successful grasps. Fig.~\ref{fig:play_steering} shows that the additional suboptimal data does not improve the base policy, but improves the \ours{} policy steered from it from $30\%$ to $64\%$. Retry and play data don't improve the pretrained policy, but they expand the behaviors available to \ours{}, resulting in better performance after post-training.

\begin{wrapfigure}[27]{r}{0.32\textwidth}
    \centering
    \vspace{-0.4em}

    \includegraphics[width=\linewidth]{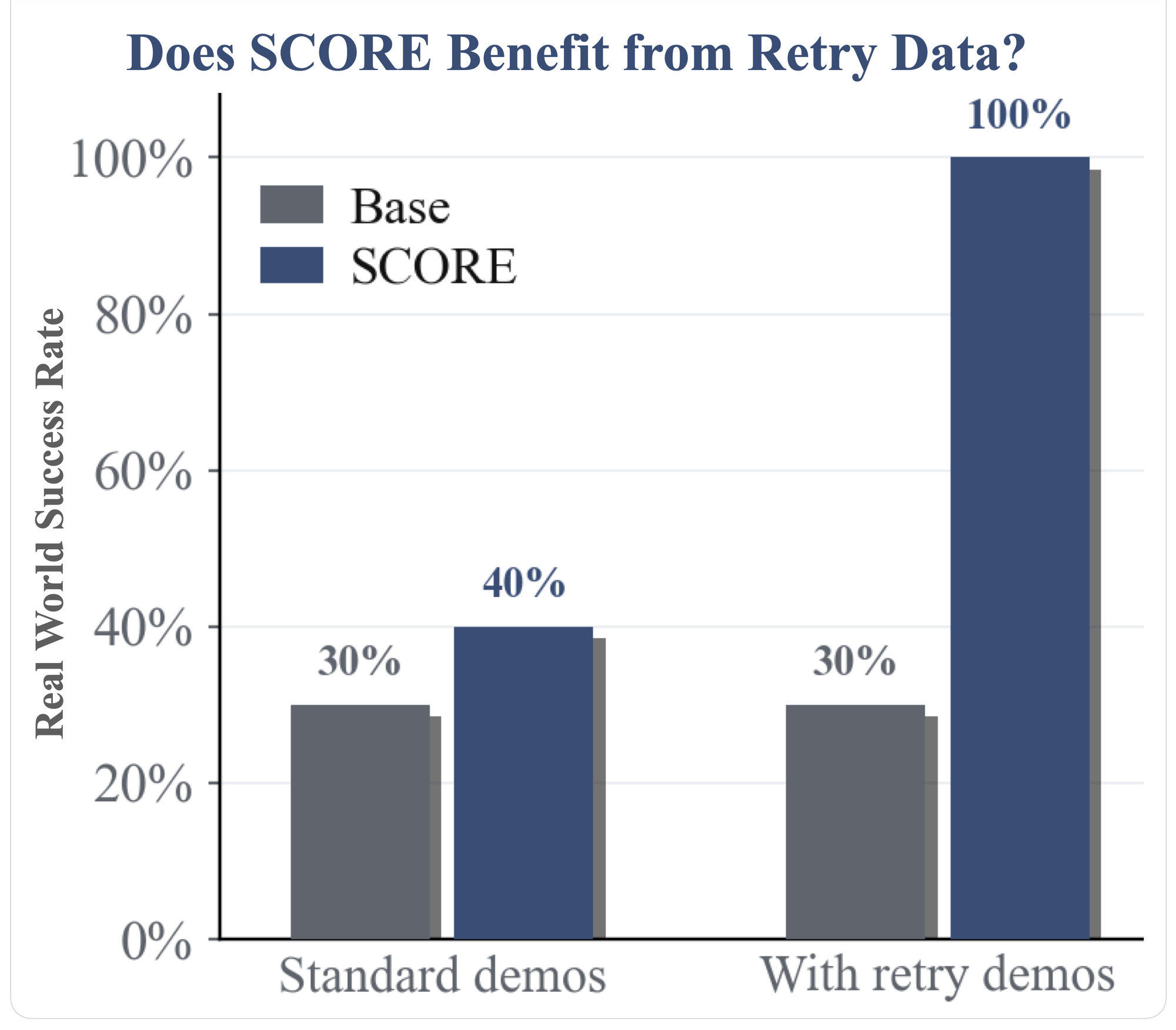}
\caption{\footnotesize{\textbf{Cube Pinch retry data.} Retry data leaves the base policy unchanged, but lets \ours{} improve from 40\% to 100\% success after simulation steering.}}

    \label{fig:retry_steering}

    \vspace{0.4em}

    \includegraphics[width=\linewidth]{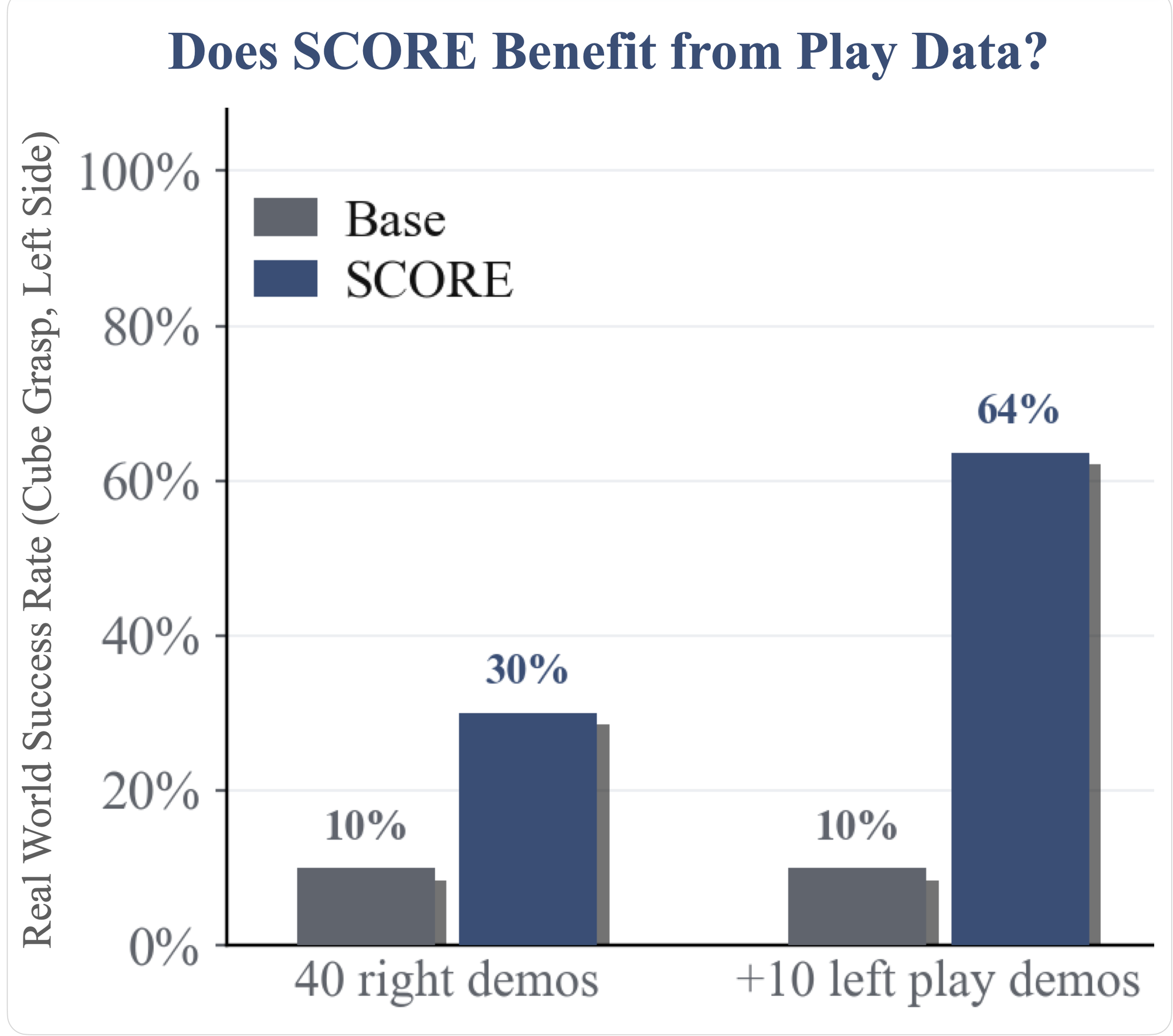}
    \caption{\footnotesize{\textbf{Play data ablation.} With only right-side coverage, \base{} and \ours{} fail on the left; adding left play data lets \ours{} improve.}}
    \label{fig:play_steering}
\end{wrapfigure}

\textbf{Can \ours{} adapt to unseen objects and distractors?} Our previous experiments test the ability of \ours{} to improve policies in fixed environments, but real world tasks are constantly changing. Below, we train \ours{} in a simulation environment unseen by the base policy, then deploy the resulting policy directly in the new environment.  

In \textbf{object transfer}, we test whether \ours{} can steer a policy trained on a different object to pinch a carrot (Fig.~\ref{fig:distractor_carrot}). We start by evaluating the Bottle Grasp \ours{} policy on the carrot. It fails to perform a precise pinch, achieving 22\% success. We then steer the bottle grasp base policy in an environment with a carrot, improving the real-world success rate to 67\%. Intuitively, the base policy occasionally performs the necessary pinch (e.g., when it grasps the bottleneck), and \ours{} can amplify this behavior when steered in the carrot environment. In contrast, the cup base policy does not support pinching, and thus steering it to pinch the carrot fails.

In \textbf{distractor adaptation}, we add distractor cubes to the bottle grasp task. The base policy and bottle grasp \ours{} policy both fail under this visual difference. When we train \ours{} with the same distractors in simulation, it improves to $56\%$ success, but only on one side of the workspace (Fig.~\ref{fig:distractor_carrot}). We suspect that the base policy is incapable of distinguishing the objects, and it is thus difficult to steer in the new setting. In addition to a diverse base policy, this demonstrates the importance of visual robustness.

\begin{figure*}[!t]
    \centering
    \vspace{-1.5em}

    \makebox[\textwidth][c]{%
        \includegraphics[
            width=\textwidth,
            keepaspectratio,
            trim={0.05in 0.05in 0.05in 0.05in},
            clip
        ]{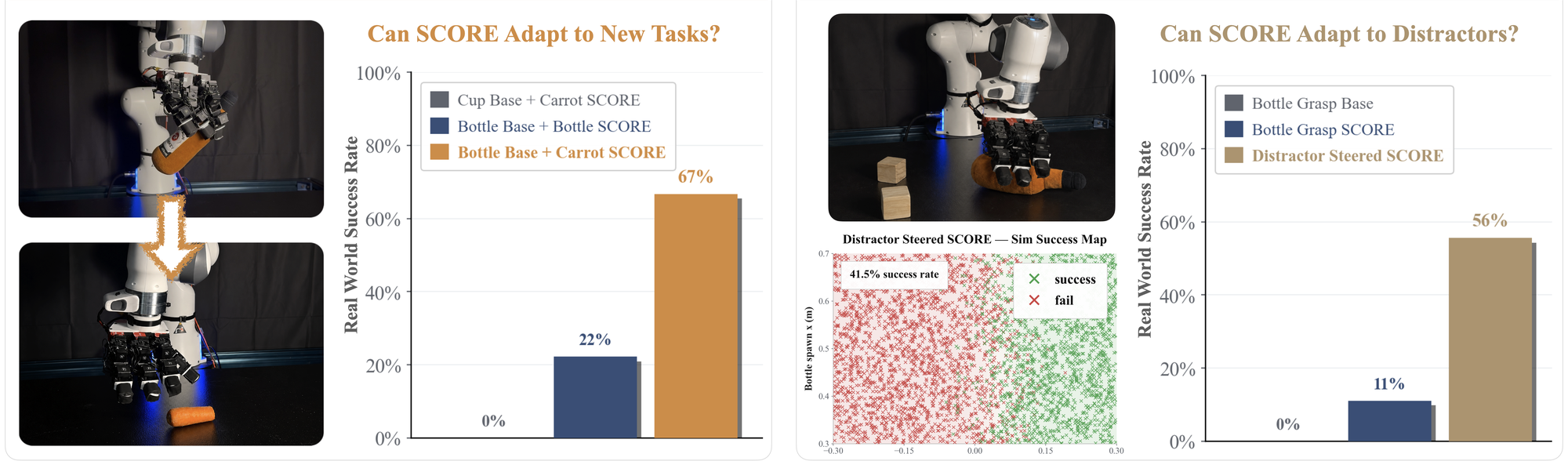}
    }
    \vspace{-1.5em}
\caption{\footnotesize{ \textbf{Adaptation experiment.}
\ours{} adapts the base policy to shifted settings only when its support contains compatible behavior. (Left) Steering the bottle-grasp prior toward carrot grasping improves real-world success from 22\% to 67\% by reusing compatible pinches already inside the prior, while the cup-grasp prior fails, as it lacks the behavior. (Right) With distractor cubes added, \ours{} improves over the broken base policy but succeeds only in part of the workspace, as steering struggles to explore far from the base policy's coverage. }}
    \vspace{-1.5em}
    \label{fig:distractor_carrot}
\end{figure*}

\textbf{Can a shared multi-task prior expand support?} To better understand how \ours{} performs on more perceptually challenging tasks, we test \ours{} in a multi task setting, which requires the policy to choose the correct grasp for different objects. We train a single base policy on the union of Cube Pinch, Bottle Grasp, and Credit Card Pick demonstrations, and train \ours{} on all three tasks simultaneously in simulation. As shown in Table~\ref{tab:multitask_success}, the multi-task base policy underperforms the single-task policies. Qualitatively, we observe grasping strategies interfere across tasks, such as applying the credit card pinch to the bottle, which we attribute to the base policy struggling to identify task structure (Appendix~\ref{app:qualitative_multitask}). After steering, however, multi-task \ours{} recovers strong single-task performance and reuses coverage across tasks: on Cube Pinch under the wider Bottle Grasp resets, it reaches $95\%$ versus $55\%$ for single-task Cube Pinch \ours{}.

Our experiments investigate the importance of dataset size, retry behaviors, suboptimal play data, multi-task priors, and shifts between the pretraining and post-training environment. In all of these experiments, \ours{} benefits from broader behavior coverage in the pretraining data, while it fails when the post-training environment is significantly out of distribution. This suggests that pretraining should aim not just for the strongest base policy, but for a broad behavior prior.

\subsection{Which Simulation Design Choices Matter for \score{}?}
\begin{wraptable}[7]{r}{0.62\textwidth}
\vspace{-\baselineskip}
\centering
\small
\setlength{\tabcolsep}{4pt}
\renewcommand{\arraystretch}{1.05}
\caption{\small \textbf{Multi-task real-world success.} ST/MT denote single/multi-task. Last row: Cube Pinch under wider Bottle Grasp resets.}
\label{tab:multitask_success}
\begin{tabular}{@{}lcccc@{}}
\hline
\textbf{Eval} & \textbf{ST Base} & \textbf{MT Base} & \textbf{ST SCORE} & \textbf{MT SCORE} \\
\hline
Bottle Grasp & 90\% & 60\% & \textbf{100\%} & \textbf{100\%} \\
Cube Pinch & 30\% & 25\% & \textbf{100\%} & \textbf{100\%} \\
Credit Card & 10\% & 10\% & \textbf{80\%} & 75\% \\
\hline
Cube (BG reset) & 10\% & 10\% & 55\% & \textbf{95\%} \\
\hline
\end{tabular}
\vspace{-\baselineskip}
\end{wraptable}

While \ours{} is a simple framework for policy improvement in any simulator, we found certain design decisions to be especially helpful in improving the sample efficiency and performance of our policies.

\textbf{Asymmetric actor-critic architecture:}
\ours{} trains the actor using deployment-compatible observations, while allowing the critic to access privileged simulation state. To isolate the effect of this design choice, we compare against a symmetric variant in which the critic receives the same point cloud and robot proprioception as the actor. The actor architecture and deployment observation space are unchanged. As shown in Fig.~\ref{fig:asymmetric_critic}, the asymmetric critic substantially improves sample efficiency and final simulated success across tasks. This suggests that privileged critic information provides a stronger learning signal during simulation training, while preserving the deployability of the actor.

\textbf{Artificial failures for recovery behavior:}
In the bottle and cube grasp tasks, \ours{} benefits significantly from failure recovery, as demonstrated in Section~\ref{subsec:retry}. However, once the policy learns reliable first-try grasps in simulation, it rarely visits failure states. Later in training, we find that this leads to forgetting of recovery behaviors. To preserve recovery behavior throughout training, we add \emph{artifical failures} to the bottle and cube task by displacing the object on a random interval after contact. Removing artifical failures lowers real-world success only from 100\% to 95\%, because most evaluation episodes are solved on the first grasp. Qualitatively, policies trained without artificial failures recover less reliably after misses, pushes, or cube displacement.

\begin{figure*}[!t]
    \vspace{-3em}

    \centering
    \includegraphics[
        width=0.86\textwidth,
        height=0.24\textheight,
        keepaspectratio,
        trim={0.15in 0.05in 0.15in 0.05in},
        clip
    ]{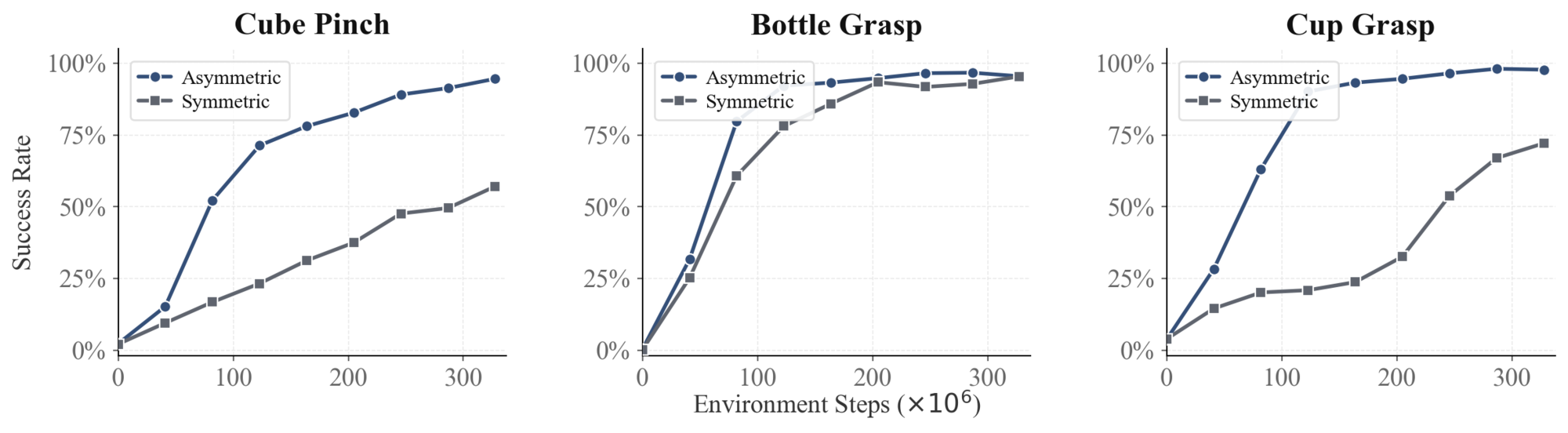}
    \vspace{-0.5em}
    \caption{\small \textbf{Asymmetric actor-critic ablation.}
    Using an asymmetric critic improves sample efficiency and final simulated success while keeping the actor observation and deployment policy unchanged. Evaluation is performed over 4096 environments every 40M steps.}
    \label{fig:asymmetric_critic}
    \vspace{-1.5em}
\end{figure*}

\section{Discussion}
We introduce \ours{}, a real-to-sim-to-real framework for off-domain policy improvement that preserves transferability by steering through the support of a real-world base policy. Across eight dexterous manipulation tasks, \ours{} improves both success rate and execution speed over the base policy and other baselines. Our experiments show that diverse behavior coverage, retry data, asymmetric actor-critic training, and sufficient dataset scale are important for robust and scalable steering. Overall, \ours{} suggests that simulation can be used to improve existing real-world policies, rather than just training new policies from scratch.

\label{sec:limitations}

\paragraph{Limitations.}
While \ours{} improves real-world policies through simulation, our experiments in Section~\ref{subsec:coverage_ablation} show that improvement remains limited in out-of-distribution environments. Furthermore, as indicated by our cross-object and play data results, scaling \ours{} to more diverse settings may be feasible with stronger and broader behavior priors. Thus, we suspect that \ours{} may substantially benefit from large-scale multi-task settings, depending on the diversity and coverage of the pretraining dataset. An exciting future direction is designing datasets and pre-training algorithms for steering, rather than zero-shot performance.

\section{Acknowledgements}

This research was supported by Amazon FAR (Frontier AI \& Robotics). We thank Entong Su for assistance with the Franka-LEAP robot setup, Patrick Yin for guidance on sim-to-real alignment, and Chuning Zhu and Jesse Zhang for helpful feedback on the manuscript.

\bibliography{bib/cites}
\newpage
\onecolumn %
\appendix
\part{Appendix} %
\parttoc %
\section{Per-Task Performance}
\label{app:addl_exps}
\label{app:task_success_table}

We report real-world success rates for each task in Table~\ref{tab:real_success}. In Fig.~\ref{fig:speed_all_tasks}, we show the per-task improvement in time to completion of \ours{} and \res{} over the base policy. We also report simulation success rates in Table~\ref{tab:sim_success_throughput}. These results show that high simulation success does not necessarily imply high real-world performance: several baselines achieve strong simulated performance but fail to transfer reliably to hardware.

\begin{figure*}[h]
    \centering
    \makebox[\textwidth][c]{%
        \includegraphics[
            width=.9\textwidth
        ]{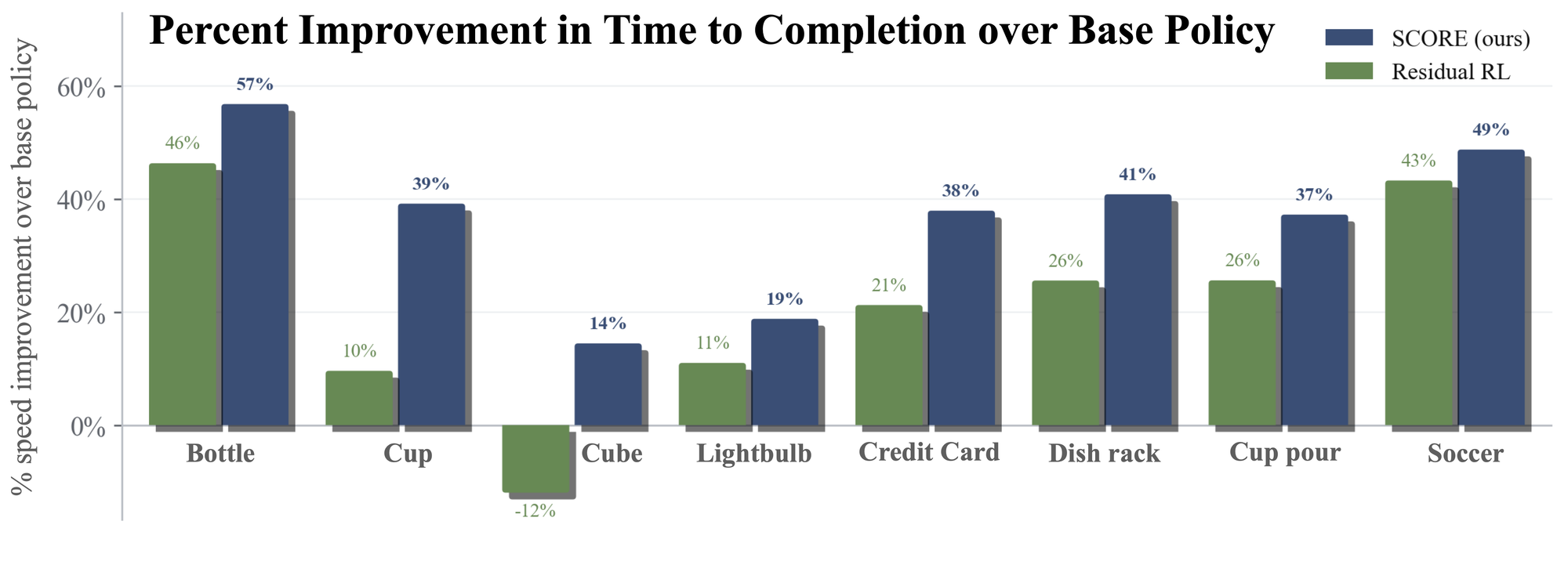}
    }
    \vspace{-1.5em}
    \caption{\footnotesize{\textbf{Percent improvement in time to completion of \ours{} and \res{} over \base}. \ours{} improves substantially over the base policy and beats \res, the nearest baseline, on all tasks. }}
    \label{fig:speed_all_tasks}
    \vspace{-1em}
\end{figure*}

\renewcommand{\arraystretch}{0.85}
\setlength{\tabcolsep}{4pt}
\begin{table*}[!ht]
\centering
\caption{\small \textbf{Simulation success rates for all tasks.}
Each cell reports success rate in simulation. RialTo success rates are
post-distillation. All evaluations are performed across 4096 environments, with environments randomized uniformly according to the bounds in Table~\ref{tab:eval_reset_randomization}. For each task, the method with the highest success rate is bolded and highlighted in green.
 Due to the extensive domain randomization in simulation, we find that simulation success rates are often lower than real world success rates, especially for transferable policies trained via \ours{}.}\label{tab:sim_success_throughput}
\begin{tabular}{llcccccc}
\toprule
Category & Task 
& \makecell{\textbf{Base}} & \makecell{\textbf{FPO}}
& \makecell{\textbf{RialTo}} & \makecell{\textbf{Res-RL}} & \makecell{\textbf{SCORE}\\\textbf{(DSRL)}} & \makecell{\textbf{SCORE}} \\
\midrule
\multirow{3}{*}{Easy}
& Bottle Grasp
& \plaincell{0.0}
& \plaincell{83.8}
& \plaincell{82.5}
& \plaincell{81.8}
& \plaincell{94.5}
& \plaincodecell{96.7} \\
& Cup Grasp
& \plaincell{3.8}
& \plaincell{57.9}
& \plaincell{97.5}
& \plaincodecell{97.9}
& \plaincell{94.2}
& \plaincell{96.5} \\
& Cube Pinch
& \plaincell{2.2}
& \plaincell{93.7}
& \plaincell{90.0}
& \plaincell{90.9}
& \plaincodecell{97.6}
& \plaincell{96.6} \\
\midrule
\multirow{3}{*}{Medium}
& Lightbulb Screw
& \plaincell{0.0}
& \plaincell{99.8}
& \plaincell{67.9}
& \plaincell{98.5}
& \plaincodecell{99.9}
& \plaincell{99.7} \\
& Dishrack Place
& \plaincell{0.0}
& \plaincell{0.0}
& \plaincodecell{67.5}
& \plaincell{47.7}
& \plaincell{64.7}
& \plaincell{65.6} \\
& Credit Card Pick
& \plaincell{4.7}
& \plaincell{68.1}
& \plaincell{45.0}
& \plaincell{78.7}
& \plaincodecell{85.4}
& \plaincell{82.2} \\
\midrule
\multirow{3}{*}{Hard}
& Ball Pour
& \plaincell{0.0}
& \plaincell{0.0}
& \plaincell{16.0}
& \plaincell{11.6}
& \plaincell{19.9}
& \plaincodecell{21.7} \\
& Soccer Push
& \plaincell{3.7}
& \plaincodecell{91.1}
& \plaincell{32.5}
& \plaincell{39.4}
& \plaincell{52.8}
& \plaincell{50.4} \\
\midrule
& \textbf{Average}
& \plaincell{1.8}
& \plaincell{61.8}
& \plaincell{62.4}
& \plaincell{68.3}
& \plaincell{76.1}
& \plaincodecell{76.2} \\
\bottomrule
\end{tabular}
\end{table*}

\renewcommand{\arraystretch}{0.85}
\setlength{\tabcolsep}{4pt}
\begin{table*}[!ht]
\centering
\caption{\small \textbf{Real-world success rate for all tasks.}
Each cell reports success rate. For each task, the method with the best success rate is bolded and highlighted in green. Evaluation is performed according to Table~\ref{tab:eval_reset_randomization}. }
\label{tab:real_success}
\begin{tabular}{llcccccc}
\toprule
Category & Task 
& \makecell{\textbf{Base}} & \makecell{\textbf{FPO}}
& \makecell{\textbf{RialTo}} & \makecell{\textbf{Res-RL}} & \makecell{\textbf{SCORE}\\\textbf{(DSRL)}} & \makecell{\textbf{SCORE}} \\
\midrule
\multirow{3}{*}{Easy}
& Bottle Grasp
& \plaincell{91.7}   %
& \plaincell{50.0}   %
& \plaincell{58.3}   %
& \plaincell{87.5}   %
& \plaincell{100.0}  %
& \plaincodecell{100.0} \\ %
& Cup Grasp
& \plaincell{46.7}   %
& \plaincell{0.0}    %
& \plaincell{53.3}   %
& \plaincell{93.3}   %
& \plaincell{86.7}   %
& \plaincodecell{100.0} \\ %
& Cube Pinch
& \plaincell{30.0}   %
& \plaincell{35.0}   %
& \plaincell{60.0}   %
& \plaincell{80.0}   %
& \plaincell{75.0}   %
& \plaincodecell{100.0} \\ %
\midrule
\multirow{3}{*}{Medium}
& Lightbulb Screw
& \plaincell{50.0}   %
& \plaincell{37.5}   %
& \plaincell{0.0}    %
& \plaincell{50.0}   %
& \plaincell{100.0}  %
& \plaincodecell{100.0} \\ %
& Dishrack Place
& \plaincell{50.0}   %
& \plaincell{0.0}    %
& \plaincell{0.0}    %
& \plaincell{45.0}   %
& \plaincell{65.0}   %
& \plaincodecell{90.0} \\ %
& Credit Card Pick
& \plaincell{10.0}   %
& \plaincell{0.0}    %
& \plaincell{10.0}   %
& \plaincell{80.0}   %
& \plaincell{80.0}   %
& \plaincodecell{80.0} \\ %
\midrule
\multirow{2}{*}{Hard}
& Ball Pour
& \plaincell{11.1}   %
& \plaincell{0.0}    %
& \plaincell{33.3}   %
& \plaincell{0.0}    %
& \plaincell{66.7}   %
& \plaincodecell{88.9} \\ %
& Soccer Push
& \plaincell{13.3}   %
& \plaincell{20.0}   %
& \plaincell{0.0}    %
& \plaincell{40.0}   %
& \plaincell{60.0}   %
& \plaincodecell{60.0} \\ %
\midrule
& \textbf{Average}
& \plaincell{37.8}   %
& \plaincell{17.8}
& \plaincell{26.9}
& \plaincell{59.5}
& \plaincell{79.2}
& \plaincodecell{89.9} \\
\bottomrule
\end{tabular}
\end{table*}

\FloatBarrier

\section{Experiment Details}
\label{app:exp_details}

\subsection{Hardware and Control Setup}
\begin{figure*}[t]
    \centering
    \makebox[\textwidth][c]{%
        \includegraphics[
            width=\textwidth
        ]{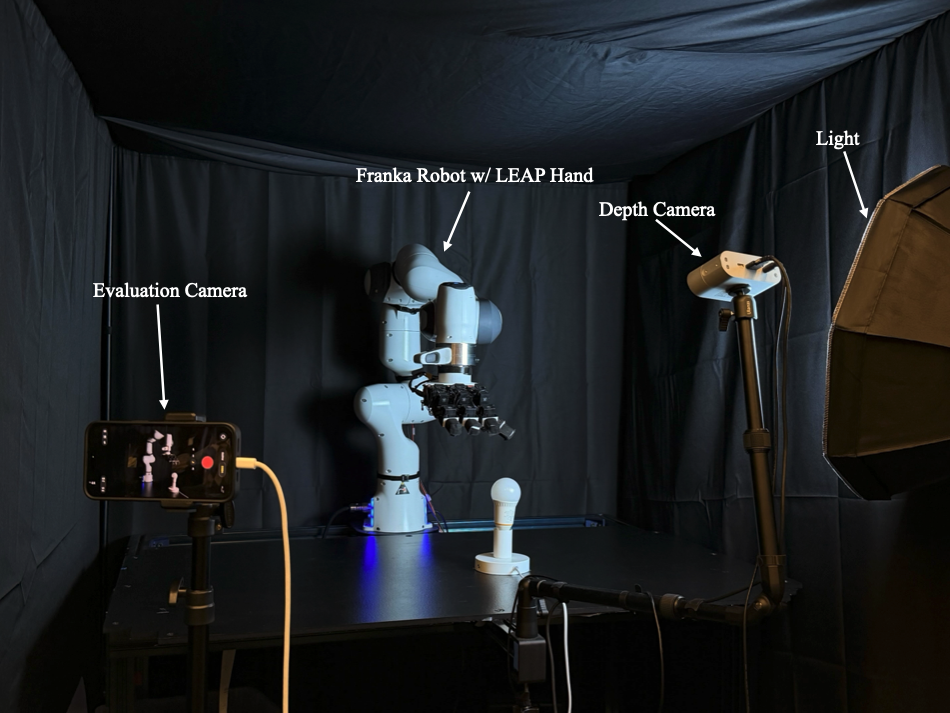}    
    }
    \caption{\small \textbf{Real-world Robot Setup}}
    \label{fig:real_world_setup}
\end{figure*}

Figure~\ref{fig:real_world_setup} shows our real-world experimental setup. We use a Franka Research 3 (FR3) robot arm equipped with a LEAP hand, constructed following the hardware setup provided by \cite{shaw2023leaphandlowcostefficient}. The workspace is observed by a single external Orbbec Femto Bolt depth camera mounted facing the front of the robot.

We deploy all policies using the DROID \cite{khazatsky2025droidlargescaleinthewildrobot} stack, which builds on Polymetis \cite{Polymetis2021}, and run closed-loop inference at 10~Hz. At each control step, the policy receives the current observation and predicts absolute joint position targets, which are executed by the robot's joint-impedance controller.

For 3D perception, we first crop the raw depth point cloud to a predefined workspace bounding box containing the table and relevant task objects. We then remove the tabletop and background points using RANSAC-based plane fitting and express the remaining points in the robot base frame. This filtered point cloud is used as the input to the policy.

\subsection{Dexterous Manipulation Tasks}
\label{app:dexterous_tasks}
Shown in Figure \ref{fig:all_tasks}, we evaluate on a suite of eight dexterous manipulation tasks designed to stress different aspects of multi-fingered control, including precise grasping, repeated in-hand motion, non-prehensile interaction, and secondary object interactions. We provide additional videos and qualitative examples on the project website:
\url{https://weirdlabuw.github.io/score/}.

\subsection{Data Collection}
\label{app:data_collection}

We collect real-world demonstrations using an Apple Vision Pro teleoperation
interface. The system tracks the operator's hand motion and end-effector motion
using keypoints, and retargets these motions to the Franka arm and LEAP hand.
This allows an operator to provide coordinated arm and finger commands for
dexterous, contact-rich manipulation tasks.

The amount of data varies by task. Table~\ref{tab:demo_counts} reports the number of real-world demonstrations collected for each task. Furthermore, for simple grasping tasks, suboptimal behavior followed by retry data was added to encourage retry behavior depicted in Figure \ref{fig:retry_data}.

\begin{figure*}[t]
    \centering
    \makebox[\textwidth][c]{%
        \includegraphics[
            width=1.1\textwidth
        ]{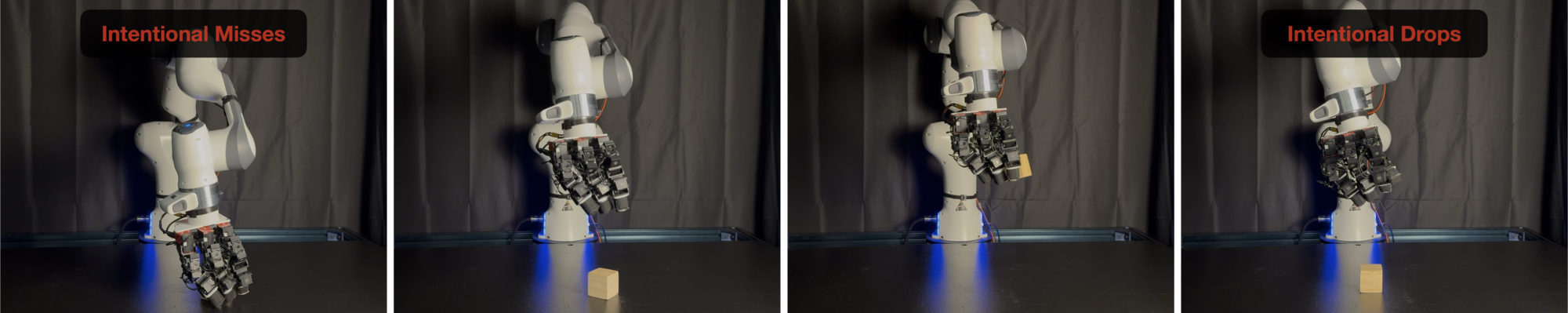}
    }
    \caption{\small \textbf{Retry Data Collection.} For the Cube Pinch and Bottle Grasp tasks, drops and misses followed by retries are captured in the dataset to encourage learning retry behavior.}
    \label{fig:retry_data}
\end{figure*}

\begin{table}[h]
    \centering
    \caption{Number of real-world demonstrations collected for each task.}
    \label{tab:demo_counts}
    \small
    \setlength{\tabcolsep}{6pt}
    \renewcommand{\arraystretch}{1.15}
    \begin{tabularx}{0.82\linewidth}{@{}l c X@{}}
        \toprule
        \textbf{Task} &
        \textbf{\# Demos} &
        \textbf{Task Type} \\
        \midrule

        Bottle Grasp &
        100 &
        Multi-finger grasping and lifting. \\

        Cup Grasp &
        100 &
        Side-grasping and lifting. \\

        Cube Pinch &
        100 &
        Precise fingertip pinch grasping. \\

        Lightbulb Screw &
        90 &
        Repeated contact and twisting. \\

        Credit Card Pick &
        130 &
        Thin-object grasping using contact with the shelf. \\

        Dishrack Place &
        100 &
        Precise object placement. \\

        Ball Pour &
        100 &
        Dynamic object manipulation with a ball. \\

        Soccer Push &
        170 &
        Non-prehensile pushing with retry behavior. \\

        \midrule
        \textbf{Total} &
        \textbf{890}  \\

        \bottomrule
    \end{tabularx}
\end{table}

\newpage
\subsection{Rewards}
\label{app:rewards}
We use minimal reward engineering across all tasks. The task-specific reward for each task is a sparse signal equal to the same success condition used for evaluation, with no dense shaping or per-task tuning. The only additional terms are lightweight joint-velocity and action-rate penalties that are shared across every task. Table~\ref{tab:reward_terms} lists the reward and success criterion for each task.

\begin{table}[ht]
\centering
\small
\setlength{\tabcolsep}{4pt}
\renewcommand{\arraystretch}{1.15}
\caption{\footnotesize Task-specific reward terms and success criteria. All tasks additionally use shared lightweight joint-velocity and action-rate penalties.}
\label{tab:reward_terms}
\begin{tabularx}{\textwidth}{@{}p{2.6cm}p{5.7cm}X@{}}
\toprule
\textbf{Task} & \textbf{Task-specific reward} & \textbf{Success criterion} \\
\midrule
Bottle Grasp &
$10.0\,r_{\mathrm{success}}$: lift bottle to target height. &
Object height $z \geq 0.25$m. \\

Cup Grasp &
$10.0\,r_{\mathrm{success}}$: lift cup to target height. &
Object height $z \geq 0.20$m. \\

Cube Pinch &
$10.0\,r_{\mathrm{success}}$: lift cube to target height. &
Object height $z \geq 0.20$m. \\

Lightbulb Screw &
$10.0\,r_{\mathrm{success}}$: rotate lightbulb to the stable on-state. &
Cumulative lightbulb rotation exceeds $2\pi$. \\

Credit Card Pick &
$10.0\,r_{\mathrm{success}}$: lift card to target height. &
Card height $z \geq z_{\mathrm{spawn}} + 0.18$m. \\

Dishrack Place &
$50.0\,r_{\mathrm{success}}$: place plate in rack slot with desired pose.
$2.0\,r_{\mathrm{return\_home}}$: return toward home configuration after placement. &
Plate is inside the rack slot with desired position and orientation. \\

Ball Pour &
$10.0\,r_{\mathrm{ball\_in\_cup}}$: place ping-pong ball inside big cup. &
Ball is inside big cup: XY distance within $0.05$m of cup center and Z position within cup height $0.12$m. \\

Soccer Push &
$10.0\,r_{\mathrm{ball\_in\_goal}}$: push ball inside goal bounding box. &
Ball is inside goal box: relative to goal frame, $x \in [-0.07,0.07]$, $y \in [-0.03,0.03]$, $z \in [0.01,0.08]$. \\
\bottomrule
\end{tabularx}
\end{table}

\subsection{Evaluation Protocol}
\label{app:evaluation_protocol}

For each task, we evaluate policies under randomized initial conditions rather
than a single fixed reset. At the beginning of each episode, task-relevant object
poses are sampled uniformly from task-specific reset ranges. 
Table~\ref{tab:eval_reset_randomization} compares the reset distribution used in
simulation with the reset distribution used during real-world evaluation.

The simulation reset distribution is generally broader than the real-world
evaluation reset distribution. In simulation, we can randomize object poses,
scene elements, and physical parameters at scale. In the real world, resets are
performed manually and are constrained by the physical workspace and by what can
be reliably initialized by the operator. Unless otherwise specified, pose offsets
are sampled uniformly from the intervals in
Table~\ref{tab:eval_reset_randomization}.

We report task-specific success rates in Table~\ref{tab:real_success}. In Fig.~\ref{fig:real_success_bar}, which shows the average real-world performance of our policy across all tasks, we use the normal approximation to compute 95\% confidence intervals.

\begin{table*}[t]
\centering
\small
\setlength{\tabcolsep}{4pt}
\renewcommand{\arraystretch}{1.15}
\begin{tabularx}{\textwidth}{@{}l X X@{}}
\toprule
\textbf{Task} &
\textbf{Simulation reset randomization} &
\textbf{Real evaluation reset randomization} \\
\midrule

Bottle Grasp &
Bottle pose is centered at $(0.50, 0.00, 0.11)$ with
$\Delta x=\pm 0.20$m, $\Delta y=\pm 0.30$m, and yaw sampled from
$\{0,\pi\}$. &
Evaluated over 24 trials. Randomization is same as simulation. \\

Cup Grasp &
Cup pose is centered at $(0.525, 0.10, 0.07)$ with
$\Delta x\in[-0.075,0.175]$m and $\Delta y=\pm 0.10$m. &
Evaluated over 15 trials. Cup is randomized only along the $x$ axis. The real
reset range is narrower than simulation, with approximately $0.075$m removed
from each side of the simulated $x$ range, and no $y$ randomization. \\

Cube Pinch &
Cube pose is centered at $(0.50, 0.00, 0.07)$ with
$\Delta x=\pm 0.20$m and $\Delta y=\pm 0.20$m. &
Evaluated over 20 trials. Randomization is same as simulation. \\

Lightbulb Screw &
Lightbulb pose is centered at $(0.545, 0.13, 0.00)$ with
$\Delta x=\pm 0.02$m, $\Delta y=\pm 0.03$m, and
$\Delta z\in[-0.01,0.02]$m. The bulb is initialized roughly $1.5$ turns away
from completion. &
Evaluated over 8 trials. Lightbulb pose is fixed in the $x$ and $y$ directions
during real evaluation; we unscrew the lightbulb one full rotation away from the
stable yellow light. \\

Dishrack Place &
Plate pose is centered at $(0.58, -0.12, 0.05)$ with
$\Delta x=\pm 0.055$m and $\Delta y\in[-0.048,0.05]$m. Rack pose is centered
at $(0.57, 0.14, 0.00)$ with $\Delta x=\pm 0.02$m and
$\Delta y=\pm 0.02$m. &
Evaluated over 20 trials. Plate pose is randomized over the same range as
simulation. The rack pose is fixed during real evaluation. \\

Credit Card Pick &
Shelf pose is centered at $(0.55, -0.20, 0.00)$ with
$\Delta x\in[-0.05,0.10]$m, $\Delta y=\pm 0.03$m, and
$\Delta z\in[-0.01,0.03]$m. The card receives an additional $x$ offset in
$[-0.05,0.00]$m on the shelf. &
Evaluated over 10 trials. The card is fixed during real evaluation; we do not
randomize the card pose. \\

Ball Pour &
Small cup pose is centered at $(0.54, -0.115, 0.07)$ with
$\Delta x=\pm 0.08$m and $\Delta y=\pm 0.04$m. Big cup pose is centered at
$(0.55, 0.12, 0.07)$ with $\Delta x=\pm 0.08$m and
$\Delta y=\pm 0.08$m. &
Evaluated over 9 trials. The source cup is randomized only along the $x$ axis.
The target cup is fixed during real evaluation. \\

Soccer Push &
Ball pose is centered at $(0.53, -0.125, 0.035)$ with
$\Delta x=\pm 0.15$m and $\Delta y\in[-0.15,0.10]$m. &
Evaluated over 15 trials. Randomization is same as simulation. \\

\bottomrule
\end{tabularx}
\vspace{-0.05in}
\caption{\small
\textbf{Simulation and real-world evaluation reset distributions.}
We evaluate policies under randomized initial conditions. All reported poses are
in the robot base frame. Simulation generally uses a broader reset distribution,
while real-world evaluation uses the same task setup as data collection with
restricted object and scene randomization.}
\label{tab:eval_reset_randomization}
\end{table*}

\newpage

\subsection{Baseline Implementations}
\label{app:baselines}

\paragraph{FPO~\cite{mcallister2025flowmatchingpolicygradients, yi2026flowpolicygradientsrobot}.}
We adapt the FPO++ implementation from~\cite{yi2026flowpolicygradientsrobot} to fine-tune our pre-trained flow matching behavior cloning policy. The PointNet observation encoder is frozen, and RL gradients are applied only to the UNet denoising network. We use the same asymmetric actor-critic setup as \ours{}: the actor receives deployment-compatible point cloud observations, while a separate MLP critic with hidden layers $[256,128,64]$ and ELU activations receives privileged simulation state.

Following FPO++, we replace the standard PPO likelihood ratio with a CFM-loss-based policy ratio,
\[
\rho =
\exp\left(
\mathcal{L}^{\mathrm{old}}_{\mathrm{CFM}}
-
\mathcal{L}^{\mathrm{new}}_{\mathrm{CFM}}
\right),
\]
where $\mathcal{L}^{\mathrm{old}}_{\mathrm{CFM}}$ is computed using the frozen policy from the previous update and $\mathcal{L}^{\mathrm{new}}_{\mathrm{CFM}}$ is computed using the current policy. We then optimize the PPO clipped surrogate objective with the asymmetric trust region variant used in the reference implementation. Unless otherwise stated, we use separate Adam optimizers for the actor and critic with learning rates $3 \times 10^{-5}$ and $1 \times 10^{-3}$, respectively, clip range $\epsilon = 0.01$, GAE with $\gamma = 0.99$ and $\lambda = 0.95$, and 4096 parallel environments with 200 rollout steps per update.

We also attempted BC regularization, which anchors the fine-tuned UNet to the frozen behavior cloning policy by penalizing deviations in CFM loss on real demonstration data. We found that this regularization introduced a sensitive tradeoff: large coefficients prevented meaningful policy improvement, while small coefficients behaved similarly to unregularized FPO. Since no tested coefficient improved the final tradeoff between simulation improvement and real-world transfer, we report unregularized FPO as the baseline. 

We swept 150 FPO configurations across the axes in Table~\ref{tab:fpo_sweep}. CFM sample counts above 8 required reducing rollout length or minibatch size to stay within memory budget. More than 5 PPO epochs per update did not improve performance despite FPO++ using 16--32 for locomotion. No amount of BC regularization coefficient improved over unregularized FPO, so we omit it from main comparisons. Our final hyperparameters are reported in Table~\ref{tab:baseline_hparams}.

\paragraph{RialTo~\cite{torne2024reconcilingrealitysimulationrealtosimtoreal}.}
We implement a RialTo-style real-to-sim-to-real baseline. For each task, we scan the deployment environment into simulation and apply inverse distillation: a perceptual policy trained on real demonstrations is rolled out in the scanned simulator, and successful rollouts are recorded with privileged state (joint state, object pose, task-specific features) to produce a privileged demonstration buffer.

We train the privileged teacher in two stages. First, we behavior-clone a state-based MLP policy on the inverse-distilled demonstrations using MSE loss. Second, we initialize a PPO policy from the BC checkpoint and fine-tune it in simulation with a BC-regularized objective:
\[
\mathcal{L}_{\mathrm{BC}} = \left\| \mu_{\pi}(s_{\mathrm{demo}}) - a_{\mathrm{demo}} \right\|_2^2,
\]
where $\mu_{\pi}$ is the Gaussian policy mean. Unlike the original RialTo, we use a continuous joint-action space. We find 30 critic-only warmup rollouts necessary to prevent collapse of the base policy before enabling policy-gradient updates. We additionally run a sweep over the BC loss coefficient, finding 10k to be the best value, as shown in Fig.~\ref{fig:distributional_tradeoff_empirical}.  All hyperparameters are reported in Table~\ref{tab:baseline_hparams}.

\paragraph{Residual RL.}
We compare against a Residual RL baseline implemented by removing the noise prediction branch from \ours{} while keeping all other components unchanged. The actor receives the same observations, the critic uses the same privileged state, and training uses the same RL objective, optimizer settings, and rollout configuration.

\clearpage

\begin{table}[h]
    \centering
    \caption{Hyperparameters for the RialTo and FPO baselines.}
    \label{tab:baseline_hparams}
    \begin{minipage}[t]{0.45\linewidth}
        \centering
        \textit{RialTo}
        \vspace{0.5em}
        \begin{tabular}{ll}
            \toprule
            \textbf{Parameter} & \textbf{Value} \\
            \midrule
            \multicolumn{2}{l}{\textit{Stage 1: Inverse Distillation}} \\
            \midrule
            Policy class          & MLP \\
            Hidden layers         & {[}512, 512, 512{]} \\
            Activation            & Tanh \\
            Action space  & Continuous joints \\
            Loss                  & MSE \\
            Optimizer             & Adam \\
            Learning rate         & $3\times10^{-4}$ \\
            LR schedule           & Cosine annealing \\
            Training epochs       & 500 \\
            \midrule
            \multicolumn{2}{l}{\textit{Stage 2: PPO Fine-tuning}} \\
            \midrule
            Policy class            & Gaussian MLP \\
            Hidden layers           & {[}512, 512, 512{]} \\
            Learning rate           & $1\times10^{-4}$ \\
            Clip range              & $0.2$ \\
            Discount $\gamma$       & $0.99$ \\
            GAE $\lambda$           & $0.95$ \\
            Value loss coefficient  & $0.1$ \\
            BC loss coefficient     & $10000$ \\
            Initial $\log\sigma$    & $-3.0$ \\
            Parallel environments   & $4096$ \\
            Rollout steps           & $200$ \\
            Batch size              & $4096$ \\
            PPO epochs per update   & $1$ \\
            Critic warmup rollouts  & $30$ \\
            \bottomrule
        \end{tabular}
    \end{minipage}
    \hspace{0.04\linewidth}
    \begin{minipage}[t]{0.45\linewidth}
        \centering
        \textit{FPO}
        \vspace{0.5em}
        \begin{tabular}{ll}
            \toprule
            \textbf{Parameter} & \textbf{Value} \\
            \midrule
            \multicolumn{2}{l}{\textit{Architecture}} \\
            \midrule
            Actor observation       & PCD (deployment) \\
            Critic observation      & Privileged state \\
            PointNet encoder        & Frozen \\
            Denoising network       & UNet (fine-tuned) \\
            Critic hidden layers    & {[}256, 128, 64{]} \\
            Critic activation       & ELU \\
            \midrule
            \multicolumn{2}{l}{\textit{Training}} \\
            \midrule
            Policy ratio            & CFM-loss-based $\rho$ \\
            Trust region            & Asymmetric PPO clip \\
            Actor optimizer         & Adam \\
            Critic optimizer        & Adam \\
            Actor learning rate     & $3\times10^{-5}$ \\
            Critic learning rate    & $1\times10^{-3}$ \\
            Clip range $\epsilon$   & $0.01$ \\
            Discount $\gamma$       & $0.99$ \\
            GAE $\lambda$           & $0.95$ \\
            Parallel environments   & $4096$ \\
            Rollout steps           & $200$ \\
            BC regularization       & None \\
            \bottomrule
        \end{tabular}
    \end{minipage}
\end{table}

\begin{table}[h]
    \centering
    \caption{FPO hyperparameter sweep}
    \label{tab:fpo_sweep}
    \begin{tabular}{ll}
        \toprule
        \textbf{Parameter} & \textbf{Values Swept} \\
        \midrule
        Actor learning rate      & $\{1\times10^{-5},\,3\times10^{-5}\}$ \\
        Critic learning rate     & $\{1\times10^{-4},\,1\times10^{-3}\}$ \\
        CFM samples per action   & $\{1,4,8,16,32\}$ \\
        PPO epochs per update    & $\{1,5,10,16,32\}$ \\
        Trust region mode        & asymmetric PPO, standard clip, CFM-loss clamp, \\
                                 & neg.-advantage clamp, ST-estimator log-ratio clamp \\
        Critic warmup updates    & $\{0,1,10,30\}$ \\
        BC regularization coeff. & $\{0,0.001,0.002,0.005,0.01,0.02,0.05,$ \\
                                 & $0.1,0.2,0.3,0.5,1.0,1.5,3.0,5.0\}$ \\
        Target KL                & $\{0.1,\,0.5\}$ \\
        Optimizer                & Adam (wd$=0$),\ AdamW (wd$=1\times10^{-6}$) \\
        \midrule
        \multicolumn{2}{l}{\textit{Fixed}} \\
        \midrule
        Clip range $\epsilon$    & $0.01$ \\
        Discount $\gamma$        & $0.99$ \\
        GAE $\lambda$            & $0.95$ \\
        Max gradient norm        & $1.0$ \\
        \bottomrule
    \end{tabular}
\end{table}

\newpage
\section{Qualitative Analysis}
\label{app:qualitative}

\subsection{Single-task policy improvement and failure modes}
\label{app:qualitative_single_task}

We provide qualitative examples of common failure modes for the baselines and our method. These examples illustrate that high success in simulation does not necessarily imply safe or reliable real-world deployment. We observe two broad failure patterns. Weakly constrained methods can improve aggressively in simulation by exploiting simulator-specific dynamics or drifting away from the real-world behavior distribution. More conservative methods can preserve transferability, but may also preserve slow, imprecise, or suboptimal behaviors from the base policy. In contrast, \ours{} targets a middle ground: it improves the policy in simulation while constraining the learned behavior to remain within the action support of the real-world base policy.

\paragraph{FPO.}
FPO often improves simulated success by moving outside the support of the real-world base policy. This can produce behaviors that appear effective in simulation but are brittle or unsafe on hardware. In the Lightbulb Screw task shown in Fig.~\ref{fig:fpo_failures}, the policy learns an abnormal twisting strategy that applies excessive force to the lightbulb fixture, causing the stand to be pulled off the table rather than reliably rotating the bulb. In Soccer Push, FPO learns to strike the table and trap the ball using simulator-specific contact dynamics. While this behavior can achieve the simulated objective, it is unsafe to deploy on real hardware and does not reflect the intended task strategy. In Credit Card Pick, the policy learns to swipe the card off the shelf, exploiting differences between simulated and real contact dynamics. Similarly, in Bottle Grasp, we observe a smaller but still noticeable drift from the base behavior: the middle finger twists inward, producing grasps that are less stable in the real world. Across these examples, FPO's failures are consistent with a lack of support preservation: the optimized policy discovers high-reward behaviors in simulation that are not reliably transferable.

\paragraph{RialTo.}
RialTo is more conservative than FPO because it explicitly regularizes policy improvement toward the base behavior distribution. This makes it less likely than FPO to discover extreme simulator-specific reward-hacking strategies, but it also limits how much the policy can improve. As shown in Fig.~\ref{fig:rialto_fails}, RialTo can still produce small but harmful deviations from real-world behavior. In Soccer Push, the fingers drag on the table and damage the pinky finger. In Bottle Grasp, RialTo also learns to invert the middle finger in an unnatural way, producing unstable grasps on hardware. In Dishrack Place, RialTo learns reward-hacking behavior in simulation by moving the dishrack center close enough to the success region defined in Table~\ref{tab:reward_terms}, without actually placing the plate inside the dishrack. These examples reflect the central tradeoff of distributional regularization: weak regularization permits unsafe drift, while strong regularization preserves suboptimal behaviors from the base policy.

\paragraph{Residual RL.}
Residual policy methods improve the base policy by learning bounded corrective actions on top of the original behavior. This can help when the base policy already chooses a reasonable strategy and only requires local refinement. However, in our experiments, failures often occur when the base policy selects a poor or suboptimal behavior mode. In these cases, the bounded residual correction cannot steer the policy toward a substantially different strategy. This limitation is most visible in tasks where the base policy has low performance, such as Ball Pour and Dishrack Place. Residual RL can improve local precision, but it remains constrained by the behavior selected by the base policy.

\paragraph{SCORE.}
\ours{} can fail when the base policy's support does not contain a successful behavior for a task, as discussed in the limitations section (Sec.~\ref{sec:limitations}). In our experiments, however, the observed failures are primarily due to task-level precision requirements and contact-rich dynamics rather than large support drift. In Credit Card Pick, small differences in contact can cause the card to fall from the shelf, and successful grasping requires precise coordination between the fingers to slide under and lift the card. In Ball Pour, slight misalignment of the cup or pouring too quickly can cause the balls to spill outside the target container. In Soccer Push, the nontrivial dynamics of the squishy real-world ball make accurate pushing difficult, and overly hard or poorly aligned pushes can result in missed goals. In Dishrack Place, failures are rare, but can occur when the robot places the plate in an incorrect slot or contacts the dishrack geometry during insertion. Overall, \ours{} failures are mostly caused by contact sensitivity and task precision, rather than simulator-specific reward hacking or unsafe behaviors outside the real-world policy support.

\begin{figure*}[h]
    \centering
    \makebox[\textwidth][c]{%
        \includegraphics[
            width=1.1\textwidth
        ]{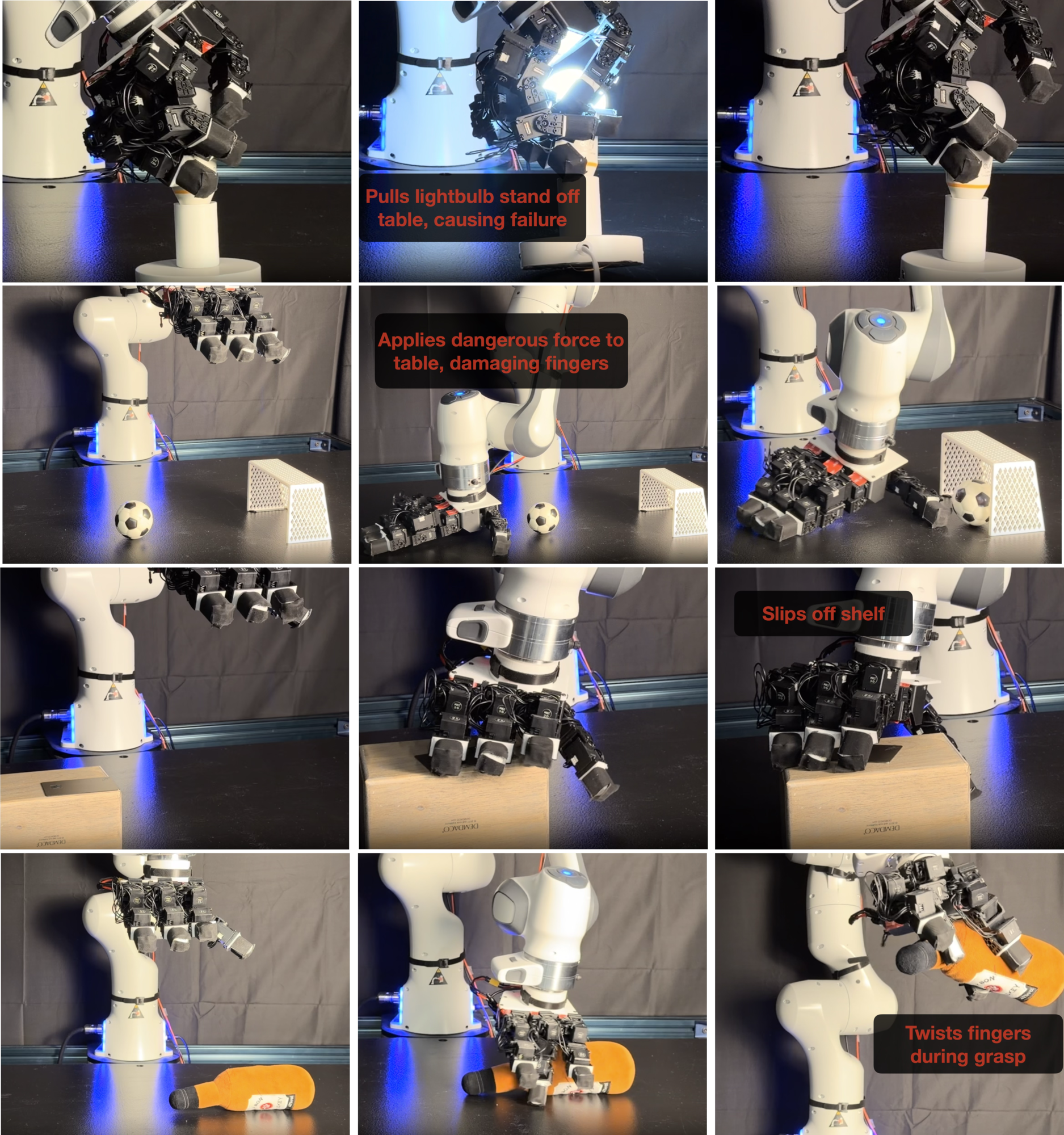}
    }
    \caption{\small \textbf{FPO failure modes.} FPO can exploit simulator-specific dynamics and drift outside the real-world policy support, producing unsafe or non-transferable behaviors.}
    \label{fig:fpo_failures}
\end{figure*}

\FloatBarrier

\begin{figure*}[t]
    \centering
    \makebox[\textwidth][c]{%
        \includegraphics[
            width=1.1\textwidth
        ]{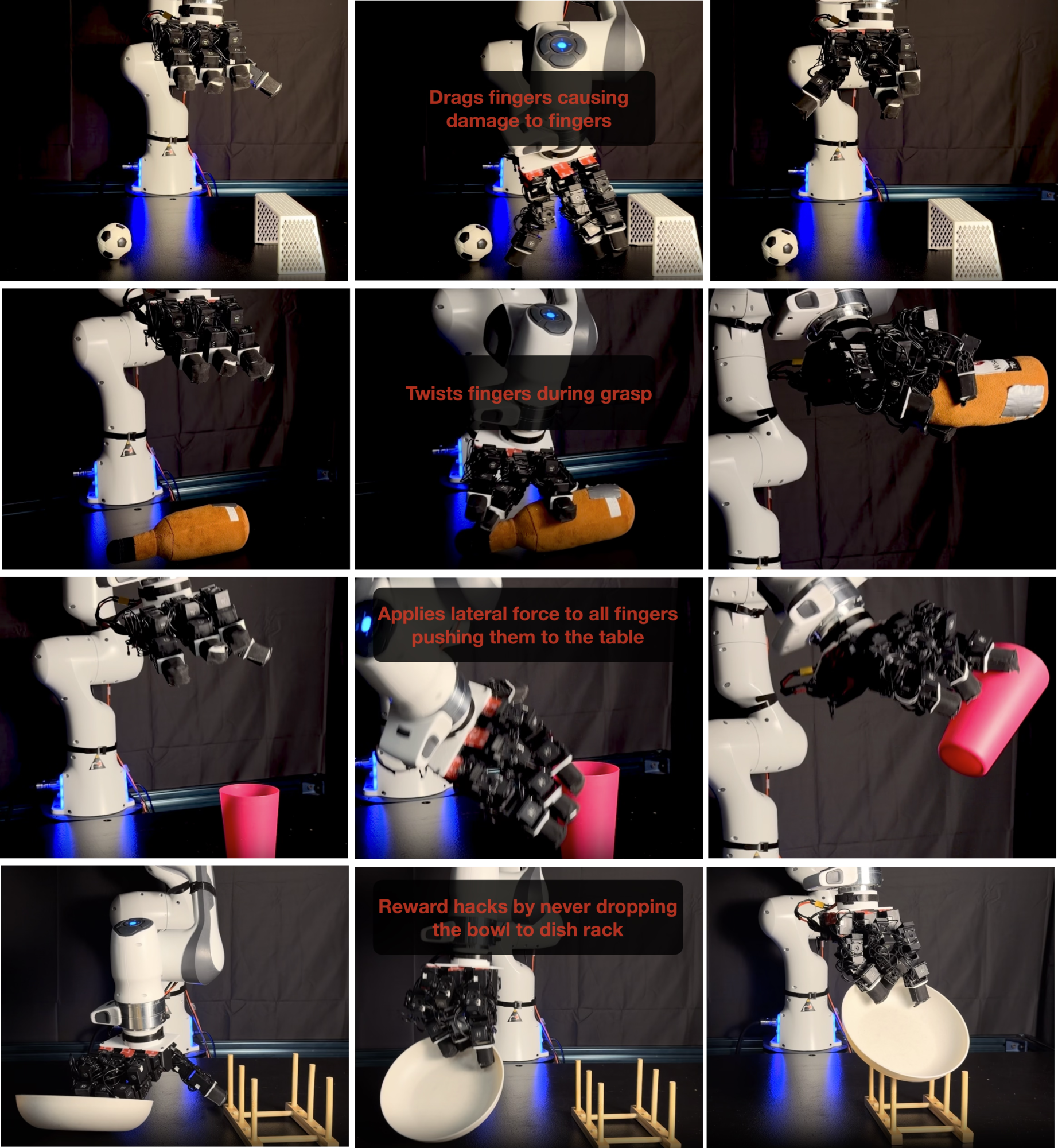}
    }
    \caption{\small \textbf{RialTo failure modes.} BC regularization can limit the amount of drift from the support of the real-world policy, but can also limit improvement and retain imprecise base-policy behaviors.}
    \label{fig:rialto_fails}
\end{figure*}

\FloatBarrier

\begin{figure*}[t]
    \centering
    \makebox[\textwidth][c]{%
        \includegraphics[
            width=1.1\textwidth
        ]{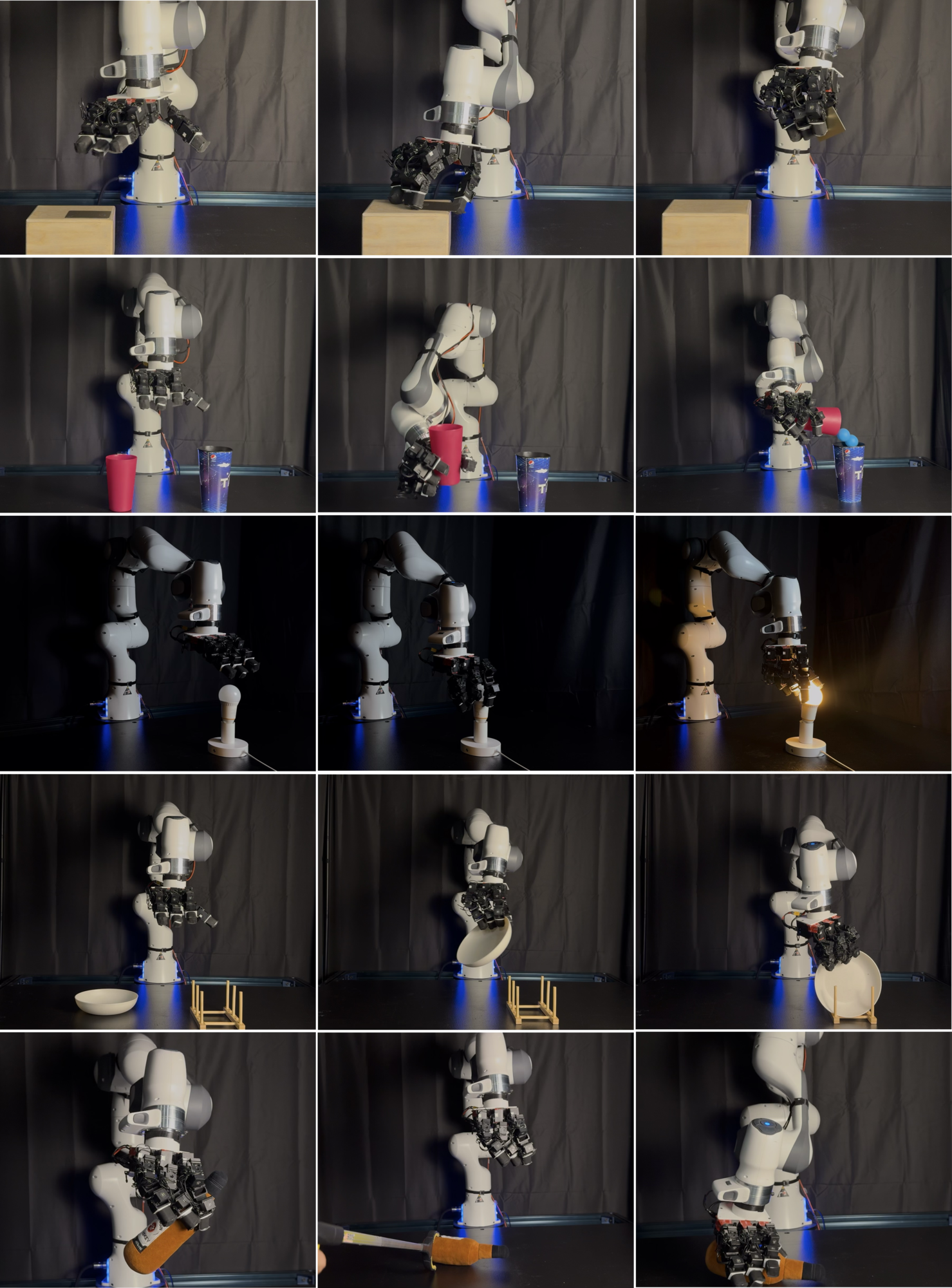}
    }
    \caption{\small \textbf{\ours{} successful rollouts.} \ours{} improves task performance while maintaining real-world-feasible behaviors within the base policy support.}
    \label{fig:code_successful}
\end{figure*}

\FloatBarrier

\begin{figure*}[t]
    \centering
    \makebox[\textwidth][c]{%
        \includegraphics[
            width=1.1\textwidth
        ]{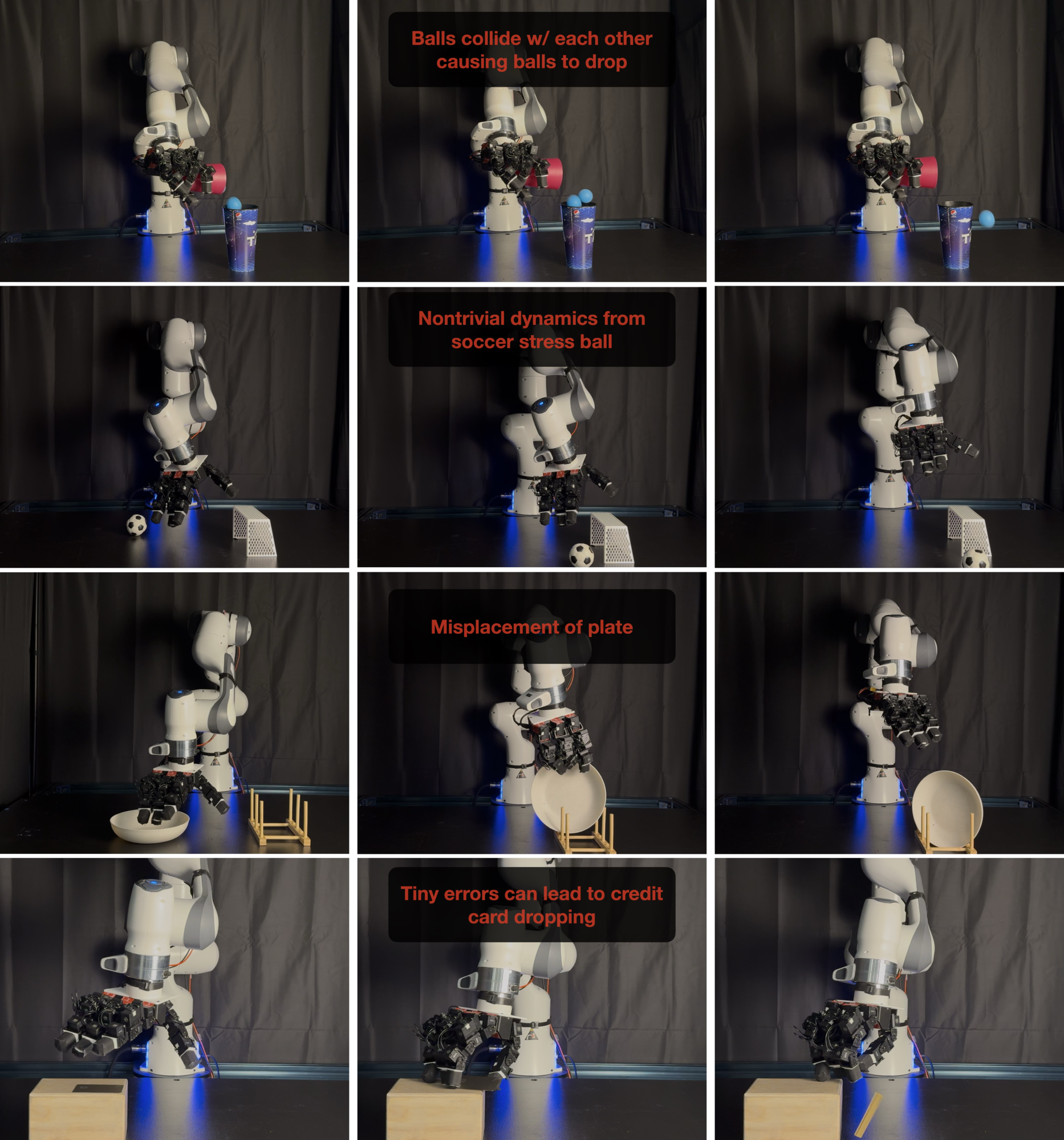}
    }
    \caption{\small \textbf{\ours{} failure modes.} \ours{} failures are primarily caused by contact sensitivity and task precision demands, rather than unsafe support drift.}
    \label{fig:code_failures}
\end{figure*}

\FloatBarrier

\subsection{Multi-task steering and cross-task behavior reuse}
\label{app:qualitative_multitask}

We also provide qualitative examples for the multi-task steering experiment. In this setting, a single base policy is trained on demonstrations from Cube Pinch, Bottle Grasp, and Credit Card Pick. This broader dataset expands the behavior support of the base policy, but also introduces interference between task-specific contact strategies. As shown in Fig.~\ref{fig:multitask_failures}, the base policy sometimes applies behaviors from one task in another setting, such as using Credit Card Pick-like high grasps for object grasping or using grasping motions that are poorly matched to the current object and reset distribution. This causes the multi-task base policy to underperform the corresponding single-task policies, despite being trained on a broader dataset.

After simulation steering, \ours{} can more reliably select task-appropriate behaviors from the shared prior. More importantly, the broader multi-task prior enables useful cross-task behavior reuse, as shown in Fig.~\ref{fig:multi_task_score}. For example, when Cube Pinch is evaluated under the wider Bottle Grasp reset distribution, single-task \ours{} often lacks sufficient support for the broader object locations, while multi-task \ours{} can reuse the wider grasping coverage learned from Bottle Grasp. Similarly, contact-rich sliding behavior from Credit Card Pick can help recover cube grasps outside the original Cube Pinch randomization. These examples suggest that multi-task data can expand the support available to \ours{}, but that steering is necessary to exploit this broader prior reliably.

\begin{figure*}[h]
    \centering
    \makebox[\textwidth][c]{%
        \includegraphics[
            width=1.1\textwidth
        ]{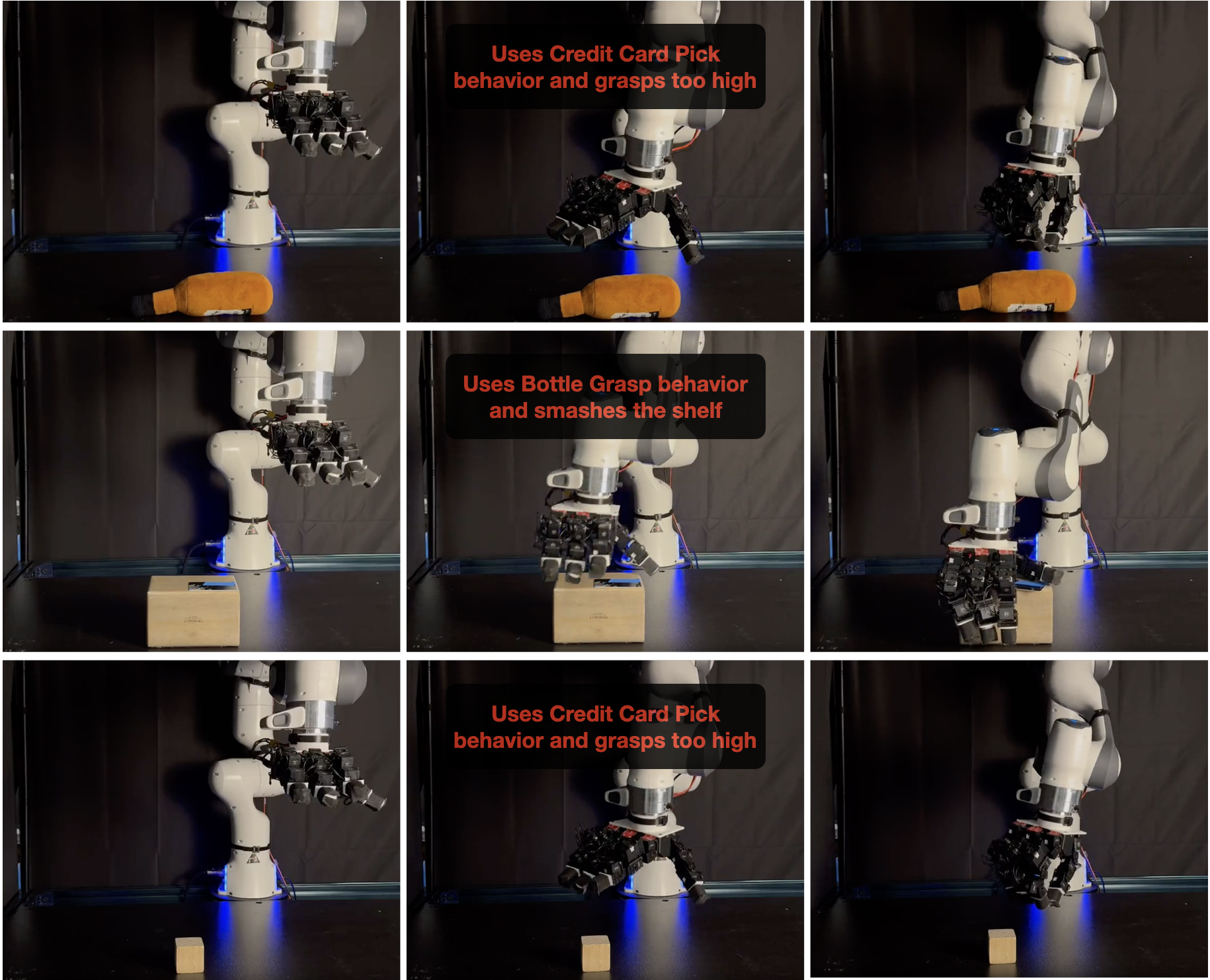}
    }
    \caption{\small \textbf{Multi-task base policy failure modes.}
Although multi-task training expands the behavior support of the base policy, direct deployment can suffer from interference between task-specific behaviors. The multi-task base policy sometimes applies behavior modes from the wrong task, such as Credit Card Pick-like high grasps during object grasping or Bottle Grasp-like motions during Credit Card Pick, leading to unstable grasps or collisions with the shelf.}
\label{fig:multitask_failures}
\end{figure*}

\begin{figure*}[h]
    \centering
    \makebox[\textwidth][c]{%
        \includegraphics[
            width=1.1\textwidth
        ]{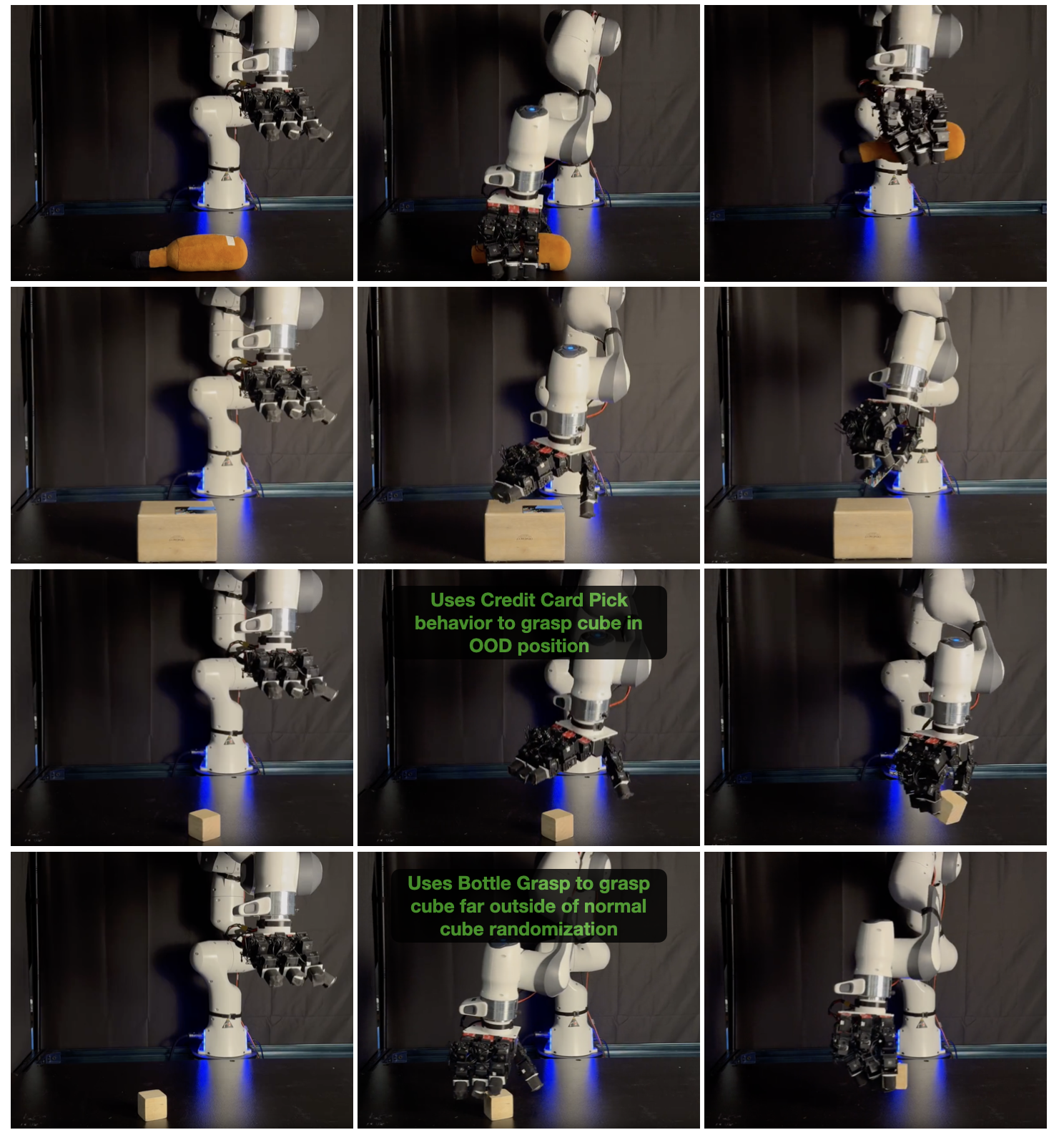}
    }
    \caption{\small \textbf{Multi-task SCORE enables cross-task behavior reuse.}
After steering a shared multi-task prior, \ours{} can select task-appropriate behaviors while also reusing useful strategies across tasks. The top two rows show successful Bottle Grasp and Credit Card Pick executions. The bottom rows show Cube Pinch under broader object placements, where multi-task \ours{} can reuse Credit Card Pick-like sliding behavior and Bottle Grasp-like wide-reset grasping behavior to recover cubes outside the original Cube Pinch randomization.}
    \label{fig:multi_task_score}
\end{figure*}

\section{Real-to-Sim-to-Real Pipeline}
\label{app:sim_details}

To ensure reproducibility, we provide full details of our real-to-sim-to-real pipeline, including model architecture, simulation design, and considerations for real world deployment. Our base policy is a conditional flow matching model trained on real-world demonstrations collected by a human. The policy uses a simple PointNet encoder together with a U-Net to parameterize the velocity field. It is conditioned on the robot joint positions and a point cloud observation of the scene, and predicts absolute joint position targets. We then run RFS in simulation using sparse rewards to improve the base policy, and finally evaluate the improved policy in the real world.

\subsection{Observation Space}
\label{app:obs_space}
To enable efficient policy improvement, our goal is to run massively parallel RL using the same observation space as the base policy. Direct vision-based observations, such as RGB or depth images, require computationally expensive rendering in simulation. In practice, we find that on a single NVIDIA L40S GPU, we can run at most 512 parallel environments with RGB or depth-image rendering when training simple MLP policies, and environment steps are significantly slower when performing rendering.

Instead, we sample points directly from the environment meshes and transform them using the ground-truth frame transforms of the robot and objects throughout the episode, as described by \cite{qi2017pointnetdeeplearningpoint}. This gives the policy access to privileged point cloud observations in simulation while preserving the same point cloud observation structure used by the real-world policy. This representation allows us to run up to 4096 parallel environments per GPU with faster environment steps, substantially reducing wall-clock training time.

The full observation consists of the absolute joint positions of the robot's 23 joints, with 7 joints for the Franka arm and 16 joints for the LEAP hand, together with a point cloud observation of the scene. In the real world, depth observations are collected using an Orbbec Femto Bolt depth camera and projected into the robot base frame to obtain a point cloud, which is cropped to the robot workspace. To account for non-Markovian effects such as robot velocity, we provide each policy with a history of the three previous proprioception observations, the three previous actions, and the current point cloud. We do not include previous point clouds in the history.

\subsection{Action Space}

The policy predicts action chunks of length 16, where each action is an absolute joint target for all 23 robot joints. These targets are passed to a joint PD controller for execution. We find that absolute joint targets transfer better to the real world because closed-loop feedback from the controller can compensate for small deviations in physical parameters, such as armature, damping, and joint friction.

\subsection{Flow Policy Architecture and Training}
\label{app:flow_policy}

We parameterize the velocity field of the flow matching model using a 1D U-Net over action sequences. The input to the U-Net is a noisy action trajectory of horizon \(H=16\), where each action has dimension 23. We condition the U-Net on three sources of information: a point cloud encoding of the scene, a history of robot joint positions, and a history of previous actions. We encode the scene using a lightweight PointNet encoder applied to a 512-point XYZ point cloud. We then concatenate the resulting global point cloud feature with the flattened proprioceptive history and action history, and inject this conditioning vector into the U-Net through FiLM conditioning.

We train the model on real-world human demonstrations using the standard CFM objective. We normalize all joint position inputs and action targets using per-dimension statistics computed from the training set, while keeping point cloud coordinates in metric units in the robot base frame. At inference time, we solve the learned flow using 5 integration steps and execute the resulting absolute joint targets through a joint PD controller.  During simulation policy improvement, and at deployment, we run closed loop inference every step, executing only the first action from the chunk.

\begin{table}[h]
    \centering
    \caption{Base flow policy architecture and training hyperparameters.}
    \label{tab:flow_policy_hparams}
    \begin{tabular}{ll}
        \toprule
        \textbf{Parameter} & \textbf{Value} \\
        \midrule
        Policy class & Conditional flow matching \\
        Action representation & Absolute joint position targets \\
        Action dimension & 23 \\
        Prediction horizon & 16 \\
        Observation history & 4 robot states \\
        Action history & 3 previous actions \\
        Point cloud size & 512 XYZ points \\
        Point cloud encoder & PointNet \\
        Sequence model & 1D U-Net \\
        U-Net channels & [256, 512, 1024] \\
        Conditioning & FiLM scale and shift \\
        Inference steps & 5 \\
        Batch size & 256 \\
        Optimizer & Adam \\
        Learning rate & \(1\times 10^{-4}\) \\
        Weight decay & \(1\times 10^{-6}\) \\
        Learning rate schedule & Cosine decay with 500 warmup steps \\
        Training epochs & 1000 \\
        Validation split & 0.05 \\
        \bottomrule
    \end{tabular}
\end{table}

\subsection{Environment Generation}

We construct simulation environments in IsaacLab to match the corresponding real-world task setups. We first combine the Franka arm and LEAP hand URDFs into a single Franka-LEAP robot model. We obtain object assets by scanning real-world objects using ARCode, an off-the-shelf object scanning application that runs on depth-enabled devices such as recent iPhone and iPad Pro models. For tasks involving defective or modified objects, we manually edit the scanned meshes in Blender. For articulated objects, we use Isaac Sim to define the corresponding joints and articulation properties.

After importing the robot and object assets, we configure the simulated scene to match the real-world environment as closely as possible, including object placement, task-relevant geometry, and articulation limits.
\subsection{Domain Randomization}

A major advantage of simulation is that we can vary physical properties, object initializations, and reset conditions at scale. To obtain policies that are even more robust in the real world, we randomize object scale, mass, friction, and spawn locations during simulation training. We also introduce additional perturbations, such as resetting objects mid episode and applying external force disturbances, to encourage recovery and retry behaviors.

Because the simulation distribution is intentionally broader and more challenging than the real-world demonstration distribution, the base policies often achieve low initial success rates in simulation despite performing substantially better in the real world. We use this gap as an opportunity for policy improvement: \ours{} exposes the base policy to difficult randomized settings and uses sparse rewards to adapt the policy toward more robust behaviors before real-world deployment.

\subsection{Asymmetric Residual Flow Steering Implementation}

Another advantage of simulation is access to ground-truth state information during training. However, directly training a fully privileged policy would require an additional distillation stage before real-world deployment. To leverage privileged information while preserving deployment compatibility, \score{} uses RFS~\cite{su2026rfsreinforcementlearningresidual} with an asymmetric actor-critic architecture.

The actor remains non-privileged and only receives observations available at deployment. In particular, we reuse the frozen point cloud encoder from the base CFM policy: the current point cloud is encoded by the frozen PointNet into a 64-dimensional feature vector, which is concatenated with a 4-step history of robot joint positions and a 3-step history of previous actions.

RFS augments the base flow policy by learning steering actions in the flow sampling process. In our implementation, the PPO actor is a single Gaussian policy over a concatenated action vector
\[
\xi = [a_r, z],
\]
where \(a_r\) denotes the residual action component, and \(z\) denotes the noise/steering component. Although we refer to these two components separately, they are produced by a single MLP actor as one concatenated vector with a single diagonal Gaussian distribution. We therefore write the actor as
\[
\xi \sim \pi_{\mathrm{steer}}(\xi \mid o),
\]
rather than as two separate policies.

The critic is privileged and is used only during simulation training. It receives simulator state features that are unavailable or unreliable in the real world such as object pose, target pose, and other task-specific state terms. We exclude raw visual observations such as point clouds or RGB images from the critic input. This asymmetric design allows the value function to use ground-truth simulator information for lower-variance advantage estimation, while keeping the actor restricted to deployment-compatible observations.

We estimate advantages using GAE with the privileged critic \(V_\psi(s)\), and update the actor using the standard clipped PPO objective:
\[
J_{\mathrm{actor}}(\theta)
=
\mathbb{E}_{(o,s,\xi) \sim \mathcal{B}}
\left[
\min\!\left(
  \rho_{\theta}\,\hat{A},\;
  \mathrm{clip}(\rho_{\theta}, 1-\varepsilon, 1+\varepsilon)\,\hat{A}
\right)
\right],
\]
where
\[
\rho_{\theta}
=
\frac{\pi_{\mathrm{steer}}(\xi \mid o)}
     {\pi_{\mathrm{steer}}^{\mathrm{old}}(\xi \mid o)}
\]
is the standard PPO importance ratio over the joint concatenated action \(\xi=[a_r,z]\), and \(\hat{A}\) is the GAE advantage estimate computed using the privileged critic. Since only the critic depends on privileged simulator state, the learned actor can be deployed directly in the real world without distillation.

\subsection{Deployment of \ours{} in the Real World}

At deployment time, we directly use the actor learned during simulation together with the original frozen base policy. The learned RFS actor predicts a concatenated steering action consisting of a noise component and a residual component at test time conditioned on the same observations as the base policy.

We use the predicted noise component as the input noise for the base CFM sampler. After integrating the learned velocity field to denoise the action trajectory, we add the predicted residual component to the resulting action. This produces a modified action trajectory that remains anchored to the base policy while allowing the RFS actor to steer the generated behavior toward higher-reward actions.

Although the base policy predicts an action chunk of horizon \(H=16\), we execute only the first action in the predicted chunk on the real robot. We then replan at the next control step using the latest observation and history.

\section{Formal Analysis of Off-Domain Policy Improvement}
\label{app:proofs}

In this section, we investigate the limitations of unconstrained RL in simulation and distributional constraints in a principled setting, formalizing the assumptions necessary to guarantee improvement of \ours{} over a real-world base policy. For consistency with the figures in this section, we note that the real-world base policy $\pi_{\mathrm{base}}$ is equivalently denoted $\pi_{\mathrm{real}}$ in some illustrations; the two refer to the same policy.

\subsection{Problem Definition}

We begin by restating Definition~\ref{def:odpi}, the off-domain policy improvement objective:
\begin{definition}[Off-Domain Policy Improvement]
Given a policy $\pi_{\mathrm{base}}$ initially learned from interaction in $\mathcal{M}_{\mathrm{real}}$,  off-domain policy improvement is the problem of finding a policy $\hat{\pi}$ such that $J_{\mathrm{real}}(\hat{\pi}) > J_{\mathrm{real}}(\pi_{\mathrm{base}})$, using additional interaction only in $\mathcal{M}_{\mathrm{sim}}$.
\end{definition}

In Section~\ref{prelim:toy_example}, through analysis of a simple toy example in this setting, we concluded that unconstrained RL and distributionally-constrained RL are not sufficient for this problem. We arrived at Equation~\ref{eq:supp_constrained_objective}, the support-constrained policy optimization objective, which we restate here:
\begin{equation}
\hat{\pi} = \arg\max_{\pi} J_{\mathrm{sim}}(\pi) \quad \text{s.t.} \quad \mathrm{supp}(\pi) \subseteq  \mathrm{supp}(\pi_{\mathrm{base}})
\end{equation}

Below, we formalize the assumptions necessary to guarantee real-world policy improvement under this objective, and show that, under these assumptions, support-constrained optimization is sufficient to recover the in-support optimum for off-domain policy improvement.

\subsection{Realizability}

In this paper, we often refer to untransferable behaviors performed by unconstrained optimization in simulation. We formalize this through \textit{realizability} below. 

\begin{definition}[Realizability]\label{def:realizable}
A state-action pair $(s, a)$ is \textbf{realizable} in the real world if $\rho^{\pi}(s, a) > 0$ for some policy $\pi$, where $\rho^{\pi}$ denotes the state-action visitation measure under $\pi$ in $\mathcal{M}_{\mathrm{real}}$. By convention, unrealizable pairs yield zero reward and result in episode termination: $r(s,a) = 0$ whenever $(s,a)$ is unrealizable. A simulation policy $\pi_\mathrm{sim}$ is \textbf{realizable} if every $(s,a)$ in its simulation visitation support is realizable; otherwise $\pi$ is \textbf{unrealizable}, or equivalently, \textbf{untransferable}. 
\end{definition}

\subsection{Limitation of Unconstrained RL}
\begin{proposition}[Unconstrained RL in Simulation Can Result in Unrealizable Policies]\label{prop:unconstrained_failure}
There exists a pair of MDPs $\mathcal{M}_{\mathrm{sim}}$ and $\mathcal{M}_{\mathrm{real}}$ such that the optimal simulation policy $\pi_{\mathrm{sim}}^* = \arg\max_\pi J_{\mathrm{sim}}(\pi)$ is unrealizable.
\end{proposition}
\begin{wrapfigure}{r}{0.45\textwidth}
    \centering
    \includegraphics[width=0.45\textwidth]{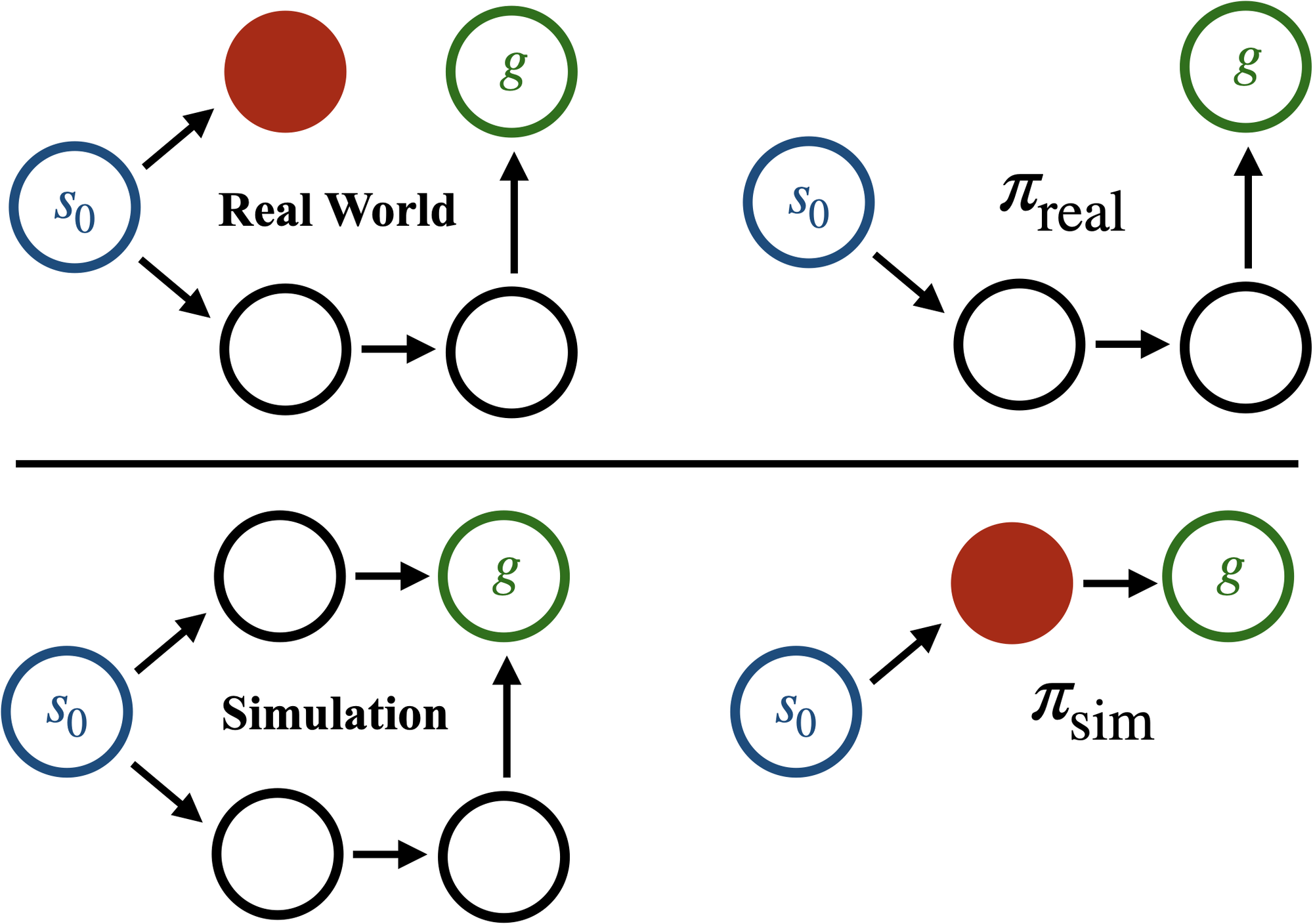}
    \vspace{-1.0em}
    \caption{\footnotesize Visual representation of Proposition~\ref{prop:unconstrained_failure}. While $\pi_\mathrm{real}$ successfully completes the task, $\pi_{\mathrm{sim}}$ exploits a transition that quickly leads to the goal in simulation, but causes the policy to get stuck when deployed in the real world. }
    \label{fig:unconstrained_discrete_mdp}
    \vspace{-0.08in}
\end{wrapfigure}

\textit{Proof Sketch}
We proceed by constructing a discrete MDP with 5 states, as shown in Fig.~\ref{fig:unconstrained_discrete_mdp}. In simulation, the shortest path to the goal passes through a state that is unrecoverable in the real world, resulting in complete failure upon deployment.

\subsection{Limitation of  Distributional Constraints}

If the goal is to find the optimal realizable policy in the real world, a natural approach is to constrain the policy improvement process in simulation to the space of realizable policies. To show that this actually results in the optimal realizable policy, we require an additional assumption: the ranking of \textit{realizable} policies is preserved between $\mathcal{M}_{\mathrm{sim}}$ and $\mathcal{M}_{\mathrm{real}}$. This is motivated by recent work showing that simulation is a reliable evaluator of real-world policies \cite{jain2025polarisscalablerealtosimevaluations}. 

\begin{assumption}[Simulation is a reliable evaluator of real world policies]\label{asm:ordering}
For any two \textbf{realizable} policies $\pi$ and $\pi'$,
\[
    J_{\mathcal{M}_{\mathrm{sim}}}(\pi) \geq J_{\mathcal{M}_{\mathrm{sim}}}(\pi')
    \implies
    J_{\mathcal{M}_{\mathrm{real}}}(\pi) \geq J_{\mathcal{M}_{\mathrm{real}}}(\pi').
\]

\end{assumption}

Under this assumption, the optimal simulation policy that is  \textit{\textbf{realizable}} in the real world must be the optimal policy in the real world (all real-world policies are realizable by definition). Thus, the off-domain policy optimization objective is:
\begin{equation}\label{eq:odpi_objective}
 \arg\max_{\pi_{\mathrm{sim}}} J_{\mathcal{M}_{\mathrm{sim}}}(\pi_{\mathrm{sim}}) \quad \text{s.t.} \quad \pi_{\mathrm{sim}} \text{ is realizable} 
\end{equation}

To actually train a policy in simulation subject to the realizability constraint, a natural approach is to regularize policy improvement against a base policy trained via real-world interaction, which is realizable by construction. In prior work, this is performed through by constraining policy improvement under a distributional distance $D$ from the real-world policy:
\begin{equation}
    \arg\max_{\pi} J_{\mathrm{sim}}(\pi) \quad \mathrm{s.t.} \quad D(\pi,\pi_{\mathrm{base}}) < \epsilon
\end{equation}
for some constant $\epsilon$. In practice, this is implemented through a behavioral cloning loss or a discriminator-based penalty~\cite{torne2024reconcilingrealitysimulationrealtosimtoreal, Peng_2021}.

Ensuring realizability under this type of regularization requires an additional assumption: the real-world prior's support must have a $\delta$-neighborhood of realizable actions, as formalized in Assumption~\ref{asm:neighborhood}. This holds whenever the prior has a margin of safety, which is usually the case, because it is difficult to teleoperate at the edge of a robot's capabilities without accidentally failing.

\begin{assumption}[Realizability neighborhood]\label{asm:neighborhood}
There exists $\delta > 0$ such that for every state $s$ and every action $a \in \text{supp}(\pi_{\mathrm{base}}(\cdot|s))$, the pair $(s, a')$ is realizable for all $a'$ with $\|a' - a\| \leq \delta$.
\end{assumption}

\begin{wrapfigure}{r}{0.45\textwidth}
    \centering
    \includegraphics[width=0.45\textwidth]{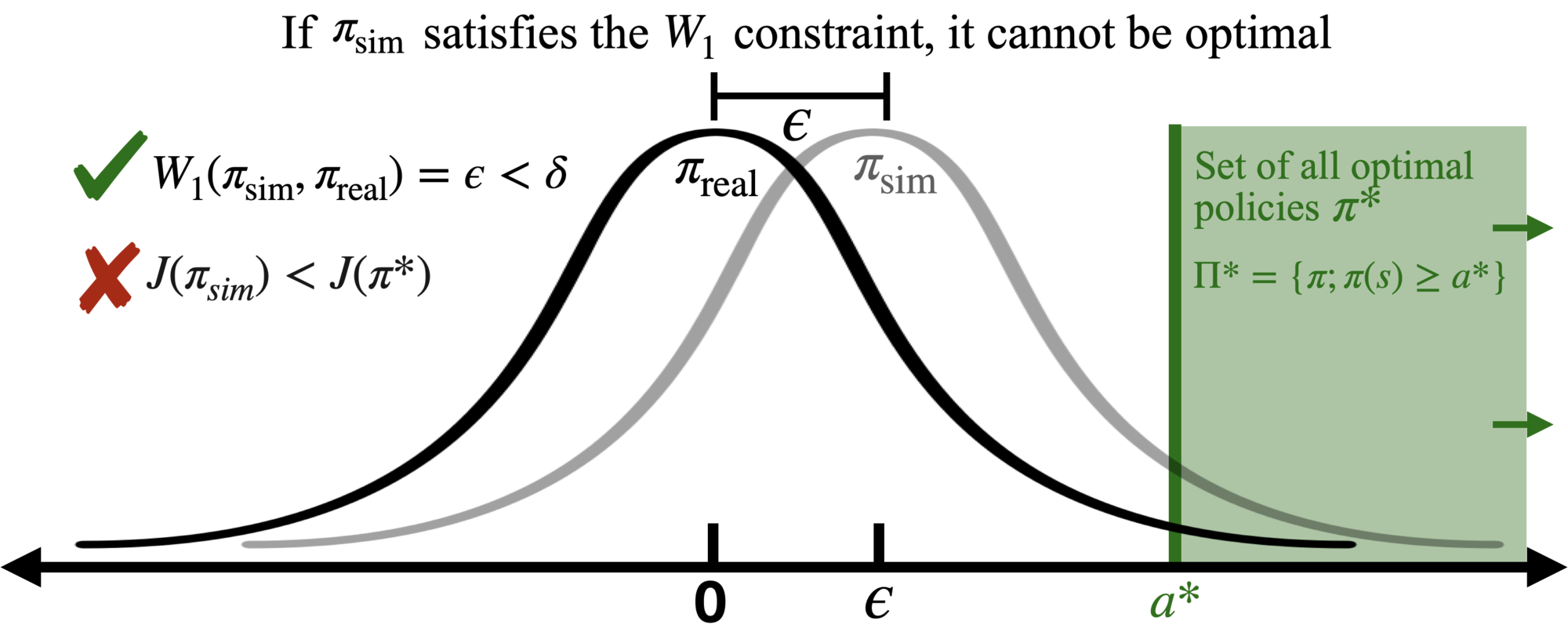}
    \vspace{-2.0em}
    \caption{\footnotesize Visual representation of Proposition~\ref{prop:dist_constraint_failure_appendix}. $\pi$ adds a residual of $\epsilon$ to $\pi_\mathrm{real}$, but this is not sufficient to recover the optimal policy}
    \label{fig:wasserstein_failure_1d}
    \vspace{-0.08in}
\end{wrapfigure}

For a small enough $\epsilon$ and large enough $\delta$,  distributional constraints ensure realizability. In practice, however, too much regularization prevents meaningful improvement, while too little allows the policy to exploit the dynamics gap. In many settings, there is no level of regularization that enables both optimality and realizability. Below, we provide a simple proof of this for the Wasserstein-1 distance, illustrated in Fig.~\ref{fig:wasserstein_failure_1d}.

\begin{proposition}[Distributional Constraints Can Preclude Optimal Improvement]\label{prop:dist_constraint_failure_appendix}
There exist MDPs $\mathcal{M}_{\mathrm{sim}}$, $\mathcal{M}_{\mathrm{real}}$, and a prior $\pi_{\mathrm{base}}$ such that the optimal realizable policy $\pi^*$ satisfies $\mathrm{supp}(\pi^*(\cdot|s)) \subseteq \mathrm{supp}(\pi_{\mathrm{base}}(\cdot|s))$ for all $s$, yet $\pi^*$ cannot be recovered by any policy satisfying $W_1(\pi(\cdot|s), \pi_{\mathrm{base}}(\cdot|s)) \leq \epsilon$. 
\end{proposition}

\textit{Proof Sketch.}
Let $H=1$, $\mathcal{S}=\{s_0\}$, $\mathcal{A}=\mathbb{R}$, $\pi_{\mathrm{base}}=\mathcal{N}(0,1)$, 
and fix any $a^* > \epsilon > 0$. Let $r(s_0, \cdot)$ be a reward function with reward 1 for all actions $a \geq a^*$, and reward 0 otherwise. The set of optimal policies $\Pi^*$ is then those whose action distribution at $s_0$ places all of its mass on $[a^*,\infty)$. Since $\mathcal{N}(0,1)$ has full support on $\mathbb{R}$, every $\pi^*\in\Pi^*$ satisfies $\mathrm{supp}(\pi^*(\cdot|s_0)) \subseteq \mathrm{supp}(\pi_{\mathrm{base}})$, so the realizability condition holds.

Let $\pi$ be an arbitrary policy satisfying $W_1(\pi(\cdot|s), \pi_{\mathrm{base}}(\cdot|s)) \leq \epsilon$. By Kantorovich-Rubinstein duality:
\[
W_1(\pi, \pi_{\mathrm{base}}) = \sup_{\|f\|_{\text{Lip}} \leq 1} \left| \mathbb{E}_{X\sim\pi}[f(X)] - \mathbb{E}_{Y\sim \mathcal{N}(0,1)}[f(Y)] \right|
\]
With $f(x)=x$, which is 1-Lipschitz:
\[
|\mathbb{E}_\pi[X] - \mathbb{E}_{\mathcal{N}(0,1)}[X]| \leq W_1(\pi, \pi_{\mathrm{base}}) \leq \epsilon
\]

Since $\mathbb{E}_{\mathcal{N}(0,1)}[X]=0$, $|\mathbb{E}[\pi]| \leq \epsilon$ for any policy that satisfies the constraint. $\min_{\pi^* \in \Pi^*} |\mathbb{E}[\pi^*]| = a^* > \epsilon$, so there is no optimal policy $\pi^*$  satisfying the constraint  $W_1(\pi^*(\cdot|s), \pi_{\mathrm{base}}(\cdot|s)) \leq \epsilon$.

\clearpage
\subsection{Support Constraints Enable Transferable Policy Improvement }
\begin{wrapfigure}{r}{0.45\textwidth}
    \centering
    \includegraphics[width=0.45\textwidth]{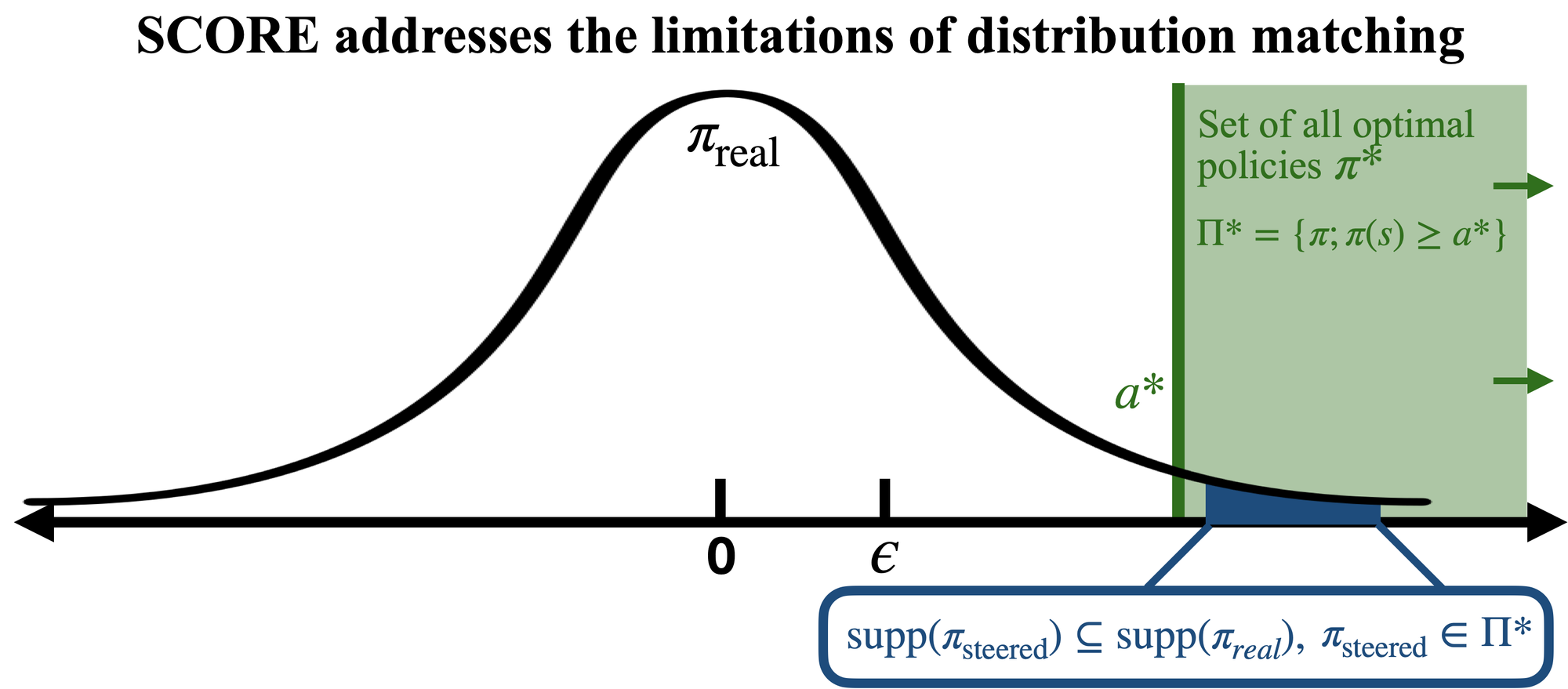}
    \vspace{-0.5em}
    \caption{\footnotesize Visual representation of how \ours{} addresses the limitations of distributional constraints shown in Fig.~\ref{fig:wasserstein_failure_1d}. By steering within the support of $\pi_{real}$, \ours{} is able to find an optimal}
    \vspace{-0.08in}
\end{wrapfigure}

The problem with distributional constraints has been extensively studied in
offline RL, where policy optimization must avoid drifting toward dangerous actions outside of a behavioral prior. To address this, prior work has
proposed \textit{support constraints}
\cite{singh2022offlinerlrealisticdatasets,
mao2023supportedtrustregionoptimization,
zhang2026reformreflectedflowsonsupport}, which require only that the learned
policy remain within the support of the behavior policy, rather than matching its
full action distribution. To apply support constraints in our setting, we must make one more assumption: 

\begin{assumption}[Realizability is closed under support]\label{asm:support_closed}
If $\pi$ is realizable and $\pi'(\cdot|s) \ll \pi(\cdot|s)$ for all $s$, then $\pi'$ is also realizable.
\end{assumption}

Intuitively, this says that a real-world prior cannot be reweighted to reach an unrealizable pair $(s,a)$. As in Assumption~\ref{asm:neighborhood}, we find that this holds in practice as long as there is a margin of safety around the data collected during teleoperation, and the dataset is large enough to account for stochasticity in the environment. Since policies within the support of a realizable policy are themselves realizable, and improvement in simulation of realizable policies guarantees improvement in the real world, we can conclude that support-constrained RL in simulation recovers the optimal policy supported by the base policy in the real world.

The proposition below states this guarantee for the idealized support-constrained case, which exactly corresponds to pure latent steering. Our RFS implementation is slightly more general because it adds a residual action after sampling from the base flow. Under Assumption~\ref{asm:neighborhood}, if this residual is bounded by the realizability margin, then RFS remains within a realizable neighborhood of the base-policy support. Thus, RFS can be interpreted as optimizing over a $\delta$-expanded support class, while pure latent steering corresponds to the special case $\delta=0$.

\begin{proposition}[Support-constrained RL recovers the in-support optimum]\label{prop:support_optimal}
Let $\pi^*_{\mathrm{base}} = \arg\max_{\pi \ll \pi_{\mathrm{base}}} J_{\mathcal{M}_{\mathrm{real}}}(\pi)$. Any support-constrained RL procedure returning $\pi_{\mathrm{supp}} = \arg\max_{\pi \ll \pi_{\mathrm{base}}} J_{\mathcal{M}_{\mathrm{sim}}}(\pi)$ yields a realizable policy with $J_{\mathcal{M}_{\mathrm{real}}}(\pi_{\mathrm{supp}}) \geq J_{\mathcal{M}_{\mathrm{real}}}(\pi^*_{\mathrm{base}})$.
\end{proposition}

\textit{Proof Sketch.}
By Assumption~\ref{asm:support_closed}, $\pi_{\mathrm{supp}} \ll \pi_{\mathrm{base}}$ is realizable; likewise $\pi^*_{\mathrm{base}} \ll \pi_{\mathrm{base}}$ is realizable. Both being realizable, Assumption~\ref{asm:ordering} applies. Optimality in $\mathcal{M}_{\mathrm{sim}}$ gives $J_{\mathcal{M}_{\mathrm{sim}}}(\pi_{\mathrm{supp}}) \geq J_{\mathcal{M}_{\mathrm{sim}}}(\pi^*_{\mathrm{base}})$, which Assumption~\ref{asm:ordering} lifts to $J_{\mathcal{M}_{\mathrm{real}}}(\pi_{\mathrm{supp}}) \geq J_{\mathcal{M}_{\mathrm{real}}}(\pi^*_{\mathrm{base}})$.

\end{document}